\def\isarxiv{1} 

\ifdefined\isarxiv
\documentclass[11pt]{article}

\usepackage[numbers]{natbib}

\else
\documentclass{article}
\usepackage{ijcai25}

\fi

\ifdefined\isarxiv

\usepackage{amsmath}
\usepackage{amsthm}
\usepackage{amssymb}
\usepackage{algorithm}
\usepackage{subfig}
\usepackage{algpseudocode}
\usepackage{graphicx}
\usepackage{grffile}
\usepackage{wrapfig,epsfig}
\usepackage{url}
\usepackage{xcolor}
\usepackage{epstopdf}

\usepackage{bbm}
\usepackage{dsfont}

\usepackage{comment}

\else

\usepackage{ijcai25}

\usepackage{times}
\usepackage{soul}
\usepackage{url}
\usepackage[hidelinks]{hyperref}
\usepackage[utf8]{inputenc}
\usepackage[small]{caption}
\usepackage{graphicx}
\usepackage{amsmath}
\usepackage{amsthm}
\usepackage{booktabs}
\usepackage{algorithm}
\usepackage{algorithmic}
\usepackage[switch]{lineno}

\linenumbers

\usepackage{comment}
\usepackage{amsfonts}
\usepackage{xcolor}
\usepackage{amssymb}

\fi
 
\allowdisplaybreaks

\ifdefined\isarxiv

\usepackage{tikz}
\usepackage{hyperref}  
\hypersetup{colorlinks=true,citecolor=blue,linkcolor=blue} 
\usetikzlibrary{arrows}
\usepackage[margin=1in]{geometry}

\else

\fi
 
\graphicspath{{./figs/}}

\theoremstyle{plain}
\newtheorem{theorem}{Theorem}[section]
\newtheorem{lemma}[theorem]{Lemma}
\newtheorem{definition}[theorem]{Definition}

\newtheorem{corollary}[theorem]{Corollary}

\newtheorem{fact}[theorem]{Fact}
\newtheorem{remark}[theorem]{Remark}

\newcommand{\R}{\mathbb{R}}

\DeclareMathOperator{\poly}{poly}

\DeclareMathOperator{\diag}{diag}

\makeatletter
\newcommand*{\RN}[1]{\expandafter\@slowromancap\romannumeral #1@}
\makeatother

\newcommand{\intdiv}{\mathbin{/\!/}}
\newcommand{\float}[2]{\left\langle #1, #2\right\rangle}

\usepackage{lineno}

\ifdefined\isarxiv

\else
\title{On the Computational Capability of Graph Neural Networks: A Circuit Complexity Bound Perspective}

\author{
    Author Name
    \affiliations
    Affiliation
    \emails
    email@example.com
}

\fi

\begin{document}

\ifdefined\isarxiv

\date{}

\title{On the Computational Capability of Graph Neural Networks: A Circuit Complexity Bound Perspective}
\author{
Xiaoyu Li\thanks{\texttt{ 7.xiaoyu.li@gmail.com}. Independent Researcher.}
\and
Yingyu Liang\thanks{\texttt{
yingyul@hku.hk}. The University of Hong Kong. \texttt{
yliang@cs.wisc.edu}. University of Wisconsin-Madison.}
\and
Zhenmei Shi\thanks{\texttt{
zhmeishi@cs.wisc.edu}. University of Wisconsin-Madison.}
\and
Zhao Song\thanks{\texttt{ magic.linuxkde@gmail.com.} The Simons Institute for the Theory of Computing at UC Berkeley.}
\and 
Wei Wang\thanks{\texttt{ wiw5087@gmail.com.} UC Berkeley.}
\and
Jiahao Zhang\thanks{\texttt{ ml.jiahaozhang02@gmail.com
.} Independent Researcher.}
}

\else

\fi

\ifdefined\isarxiv
\begin{titlepage}
  \maketitle
  \begin{abstract}
Graph Neural Networks (GNNs) have become the standard approach for learning and reasoning over relational data, leveraging the message-passing mechanism that iteratively propagates node embeddings through graph structures. While GNNs have achieved significant empirical success, their theoretical limitations remain an active area of research. Existing studies primarily focus on characterizing GNN expressiveness through Weisfeiler-Lehman (WL) graph isomorphism tests. In this paper, we take a fundamentally different approach by exploring the computational limitations of GNNs through the lens of circuit complexity. Specifically, we analyze the circuit complexity of common GNN architectures and prove that under constraints of constant-depth layers, linear or sublinear embedding sizes, and polynomial precision, GNNs cannot solve key problems such as graph connectivity and graph isomorphism unless $\mathsf{TC}^0 = \mathsf{NC}^1$. These results reveal the intrinsic expressivity limitations of GNNs behind their empirical success and introduce a novel framework for analyzing GNN expressiveness that can be extended to a broader range of GNN models and graph decision problems.

  \end{abstract}
  \thispagestyle{empty}
\end{titlepage}

{\hypersetup{linkcolor=black}
\tableofcontents
}
\newpage

\else

\maketitle

\begin{abstract}
    Graph Neural Networks (GNNs) have become the standard approach for learning and reasoning over relational data, leveraging the message-passing mechanism that iteratively propagates node embeddings through graph structures. While GNNs have achieved significant empirical success, their theoretical limitations remain an active area of research. Existing studies primarily focus on characterizing GNN expressiveness through Weisfeiler-Lehman (WL) graph isomorphism tests. In this paper, we take a fundamentally different approach by exploring the computational limitations of GNNs through the lens of circuit complexity. Specifically, we analyze the circuit complexity of common GNN architectures and prove that under constraints of constant-depth layers, linear or sublinear embedding sizes, and polynomial precision, GNNs cannot solve key problems such as graph connectivity and graph isomorphism unless $\mathsf{TC}^0 = \mathsf{NC}^1$. These results reveal the intrinsic expressivity limitations of GNNs behind their empirical success and introduce a novel framework for analyzing GNN expressiveness that can be extended to a broader range of GNN models and graph decision problems.

\end{abstract}

\fi

\section{Introduction}
Graphs are ubiquitous representations for relational data, describing interconnected elements with interactions in domains such as molecules~\cite{gsr17,wwc22}, social networks~\cite{fml19,sly21}, and user-item interactions in recommendation systems~\cite{wxw19,hdw20}. Graph Neural Networks (GNNs)~\cite{kw17,hyl17,vcc18} have emerged as a dominant tool for learning expressive representations from graphs, enabling tasks like node property prediction~\cite{wsz19,bay22}, link prediction~\cite{zc18,zzx21}, and graph classification~\cite{zcn18,xhl18}.
The key to GNNs' success lies in the message-passing mechanism~\cite{gsr17}, which aggregates information from local neighborhoods through a graph convolution matrix, followed by nonlinear transformations to update node representations.

Despite their empirical success, concerns regarding the computational limitations of GNNs are emerging~\cite{gjj20,ww22}. A fundamental research question arises from the concern: 
\begin{center}
    {\it What computational capabilities do GNNs and their variants possess, and what classes of problems can they provably solve?}
\end{center}
Addressing these questions is crucial for understanding GNNs from a principled and theoretically robust perspective, identifying their limitations, and building trust in their deployment for real-world applications.

Previous work has made significant strides in addressing these questions. A notable line of research connects GNN expressiveness to the Weisfeiler-Leman (WL) graph isomorphism test~\cite{xhl18,mrf19,mrm20,zgd24}, which iteratively refines color labels on $k$-tuples of nodes (known as $k$-WL) and distinguishes graphs by the resulting color histograms. For example, the Graph Isomorphism Network (GIN)~\cite{xhl18} equipped with summation aggregator and injective readout function is as expressive as $1$-WL, while $k$-GNNs~\cite{mrf19} achieve expressiveness equivalent to $k$-WL by encoding $k$-node subgraphs as hypernodes in message passing. However, these results focus on graph isomorphism tasks and do not address a broader range of graph query problems. Moreover, the WL framework only provides a specific level of expressiveness without establishing the theoretical upper bounds, and it often ignores the interplay between node features and graph topology, which is central to GNNs.

In this paper, we take a fundamentally different approach to analyzing the computational limitations of GNNs by examining the \textit{circuit complexity bounds}. Circuit complexity~\cite{haf22,ruz81,vol99} is a foundational topic in theoretical computer science, characterizing computational models based on the types of Boolean gates they use and the depth and size of their circuits. Importantly, circuit complexity bounds provably reflect the set of problems solvable within a given model. For instance, models (e.g., Transformers) bounded by the $\mathsf{TC}^0$ complexity class can only solve $\mathsf{TC}^0$ problems (e.g., Dyck language recognition) but cannot solve $\mathsf{NC}^1$-complete problems like arithmetic formula evaluation unless $\mathsf{TC}^0 = \mathsf{NC}^1$~\cite{chi24,ms23,ms24}. Analogously, if we demonstrate that GNN computations lie within a particular circuit complexity class, we can formally identify problems that GNNs can solve and cannot solve.

Our approach evaluates the circuit complexity of GNN components, from basic activation functions to the entire graph convolution process. We show that GNNs with a constant number of layers, $\poly(n)$ precision, and embedding sizes $d = O(n)$ can be approximated by uniform $\mathsf{TC}^0$ circuits. Consequently, unless $\mathsf{TC}^0 = \mathsf{NC}^1$, such GNNs cannot solve problems like graph connectivity problems or graph isomorphism problems. These findings illuminate the fundamental expressivity limitations of GNNs despite their empirical success and also establish a novel framework for analyzing their expressiveness, which can be seamlessly generalized to more GNN models and decision problems on graphs. 

Our contributions are summarized as follows:
\begin{itemize}
    \item We prove that a uniform $\mathsf{TC}^0$ circuit family can approximate GNNs with a constant number of layers, $\poly(n)$ precision, and $d = O(n)$ embedding size  (Theorem~\ref{thm:complexity_gnn}).
    \item We establish that unless $\mathsf{TC}^0 = \mathsf{NC}^1$, graph neural networks with a constant number of layers, $\poly(n)$ precision and $d = O(n)$ embedding size cannot solve the graph connectivity problems (Theorem~\ref{thm:ustconn_hard} and Theorem~\ref{thm:uconn_hard}).
    \item  We establish that unless $\mathsf{TC}^0 = \mathsf{NC}^1$, graph neural networks with a constant number of layers, $\poly(n)$ precision and $d = O(n)$ embedding size cannot solve the graph isomorphism problems (Theorem~\ref{thm:gi_hard}).
\end{itemize}

\noindent \textbf{Roadmap.} Section~\ref{sec:rel} reviews relevant works. Section~\ref{sec:prelim} introduces GNNs and basic concepts from theoretical computer science. Section~\ref{sec:complexity} presents the circuit complexity bounds for GNNs. Section~\ref{sec:hard} discusses the hardness of specific graph problems. Finally, Section~\ref{sec:conclusion} concludes our findings.
\section{Related Work}\label{sec:rel}
We present related works on the computational limitation of GNNs and existing circuit complexity bounds for neural networks. 

\subsection{Limitations of Graph Neural Networks} 
Graph Neural Networks (GNNs) have demonstrated impressive performance on graph learning and mining tasks. However, their inherent limitations in solving decision problems on graphs remain an open question. The predominant framework for analyzing GNN limitations is the Weisfeiler-Lehman (WL) hierarchy~\cite{lw68,sat20,zgd24}, a well-established tool for assessing GNNs' ability to address the graph isomorphism problem—an NP-intermediate problem not solvable in polynomial time~\cite{bab16,gs20}. The WL hierarchy leverages the computationally efficient heuristic of color refinement to bound the capability of differentiating non-isomorphic graphs.

The expressiveness of message-passing GNNs is bounded by the $1$-WL test~\cite{xhl18}. Standard GNN models such as GCN~\cite{kw17}, GAT~\cite{vcc18}, and GIN~\cite{xhl18} are either equivalent to or strictly limited by the expressiveness of $1$-WL. To go beyond $1$-WL, high-order GNNs extend message-passing mechanisms to $k$-node subgraphs~\cite{mrf19,mrm20,mbh19}, mimicking the $k$-WL or $k$-FWL (Folklore WL) tests. Models like $k$-GNN and $k$-FGNN match the expressiveness of these higher-order tests, offering stronger guarantees than standard message-passing GNNs. However, the parameter $k$ cannot be scaled to sufficiently large values due to inherent computational and memory constraints.

Another promising line of research involves subgraph GNNs, which aim to address the inherent symmetry of graphs that cannot be distinguished by WL tests. These models transform the original graph into a set of slightly perturbed subgraphs, which are then processed by GNNs~\cite{cmr21,pw22,qrg22}. Recent work has shown that for an $n$-node graph, subgraph GNNs operating on subgraphs with $k$ nodes are strictly bounded by $(k+1)$-FWL~\cite{qrg22}. Besides, distance-aware GNNs inject distance information—overlooked by both message-passing GNNs and $1$-WL—into their architectures. For instance, $k$-hop MPNNs aggregate information from $k$-hop neighbors in each layer and have been shown to be strictly bounded by $2$-FWL~\cite{fcl22}. Additionally, the subgraph WL hierarchy demonstrates that distance encoding can be represented by local $2$-WL tests~\cite{zfd23}.

Despite the widespread use of the WL framework in analyzing GNNs' computational limitations, it often overlooks the role of node features~\cite{brh21} and focuses exclusively on the graph isomorphism problem, making it insufficiently comprehensive and unsuitable for generalizing to other graph decision problems. A detailed comparison of our work with WL-based GNN expressiveness is provided in Section~\ref{sec:hard_discuss}.

\subsection{Circuit Complexity and Neural Networks}  
Circuit complexity, a foundational area in theoretical computer science, studies the computational power of Boolean circuit families. Various circuit complexity classes play a significant role in analyzing machine learning models. For instance, the class $\mathsf{AC}^0$ represents problems solvable in parallel using standard Boolean gates, while $\mathsf{TC}^0$ extends this by incorporating threshold or modulo gates. The stronger class $\mathsf{NC}^1$ corresponds to problems solvable by circuits with $O(\log n)$ depth and bounded fan-in~\cite{mss22}. A key result relevant to machine learning is the complexity inclusion $\mathsf{AC}^0 \subset \mathsf{TC}^0 \subseteq \mathsf{NC}^1$, though whether $\mathsf{TC}^0 = \mathsf{NC}^1$ remains an open question~\cite{vol99,ab09}.

Circuit complexity bounds have been effectively used to analyze the computational power of various neural network architectures. For example, Transformers, including two canonical variants—Average-Head Attention Transformers (AHATs) and SoftMax-Attention Transformers (SMATs)—have been studied in this context. Specifically,~\cite{mss22} shows that AHATs can be simulated by non-uniform constant-depth threshold circuits in $\mathsf{TC}^0$, while~\cite{lag23} demonstrates that SMATs can also be simulated in a $L$-uniform manner within $\mathsf{TC}^0$. A follow-up study~\cite{ms24} unifies these results, concluding that both AHATs and SMATs are approximable by $\mathsf{DLOGTIME}$-uniform $\mathsf{TC}^0$ circuits. Beyond standard Transformers, RoPE-based Transformers~\cite{sal24}, a widely adopted variant in large language models, have also been analyzed using circuit complexity frameworks~\cite{lls24,cll+24}. Similarly, the emerging Mamba architecture~\cite{gd23} falls within the $\mathsf{DLOGTIME}$-uniform $\mathsf{TC}^0$ family~\cite{cll24}. Additionally, Hopfield networks, initially introduced as associative memory systems, have also been shown to exhibit $\mathsf{TC}^0$ circuit complexity bounds~\cite{lly24}.

Despite the success of circuit complexity in analyzing other neural networks, its application to GNNs is underexplored. While some prior works have attempted to characterize the computational power of GNNs within circuit computation models~\cite{bkm20,cws24}, these efforts are orthogonal to our contributions. A detailed discussion is provided in Section~\ref{sec:hard_discuss}. 
\section{Preliminary}\label{sec:prelim}

This section provides fundamental definitions for this paper. 
We first introduce some notations. In Section~\ref{sec:float_num}, we present an in-depth overview of the computation of floating point numbers. In Section~\ref{sec:pre_gnn}, we review several basic definitions of the Graph Neural Networks (GNNs). In Section~\ref{sec:pre_comp_class}, we present some basic concepts of circuit families and their complexity. 

\paragraph{Notations.}
For any positive integer $n$, we let $[n]: = \{0, 1, 2, \dots, n\}$ denote the set of first $n$ natural numbers. For a vector $x$, $\mathrm{diag}(x)$ denotes the diagonal matrix with entries from $x$. The concatenation operator $||$ represents combining two vectors or matrices along their last dimension. For example, given $A \in \mathbb{R}^{n \times d_1}$ and $B \in \mathbb{R}^{n \times d_2}$, the result $A || B$ has shape $\mathbb{R}^{n \times (d_1+d_2)}$. We use $\boldsymbol{1}_n$ to represent the $n$-dimensional all-ones vector. Let $\mathcal{G} = (\mathcal{V}, \mathcal{E})$ be an undirected graph with vertex set $\mathcal{V} = \{v_1, \dots, v_n\}$ and edge set $\mathcal{E} = \{e_1, \dots, e_m\}$, where $n$ and $m$ are the numbers of vertices and edges, respectively. The adjacency matrix $A \in \{0, 1\}^{n \times n}$ encodes $\mathcal{E}$, with $A_{i,j} = 1$ if $(i, j) \in \mathcal{E}$, and $A_{i,j} = 0$ otherwise. The Kleene closure of a set $V$, denoted $V^*$, is the set of all finite sequences whose elements belong to $V$. Specifically, $x \in V^*$ implies $x$ is a sequence of arbitrary length with elements from $V$. 
For convenience, we denote the floating point number as $\mathsf{FPN}$.

\subsection{Floating Point Numbers}\label{sec:float_num}
In this subsection, we introduce fundamental definitions of floating-point numbers and their operations. These concepts establish a critical computational framework for implementing GNNs on real-world machines. 

\begin{definition}[Floating Point Numbers ($\mathsf{FPN}$s), Definition 9 in~\cite{chi24}]\label{dfn:fp_num}  
A $p$-bit floating-point number $(\mathsf{FPN})$ is represented as a $2$-tuple of binary integers $\langle s, e \rangle$, in which the significand $|s| \in \{0\} \cup [2^{p-1}, 2^p)$ and the exponent $e \in [-2^p, 2^p - 1]$. The value of the $\mathsf{FPN}$ is given by $s \cdot 2^e$. Specifically, when $e = 2^p$, the floating-point number represents positive or negative infinity, depending on the sign of the significand $m$. We denote the set of all the $p$-bit $\mathsf{FPN}$s as $\mathbb{F}_p$.  
\end{definition}

\begin{definition}[Rounding, Definition 9 in~\cite{chi24}]\label{dfn:rounding}
    Let the real number $r\in\R$ be of infinite precision. The closest-$p$ bit precision $\mathsf{FPN}$ for $r$ is denoted by $\mathrm{round}_p(r)\in\mathbb{F}_p$. If two such numbers exist, we denote $\mathrm{round}_p(r)$ as the $\mathsf{FPN}$ with even significand. 
\end{definition}

Building upon the fundamental concepts introduced above, we present the critical floating-point operations used to compute the outputs of graph neural networks.   
 
\begin{definition}[$\mathsf{FPN}$ operations, page 5 on~\cite{chi24}]\label{dfn:fp_ops}
    Let $x, y$ be two integers. We first denote the integer division operation $\intdiv$ as:
    \begin{align*}
  x \intdiv y &:= \begin{cases}
    x/y & \text{if $x/y$ is a multiple of $1/4$} \\
    x/y + 1/8 & \text{otherwise.}
  \end{cases}
\end{align*}
Given two $p$-bits $\mathsf{FPN}$s $\float{s_1}{e_1}, \float{s_2}{e_2} \in\mathbb{F}_p$, we define the following basic operations: 
    \begin{align*}
\float{s_1}{e_1} + \float{s_2}{e_2} &:= \begin{cases}
\mathrm{round}_p({\float{s_1 + s_2 \intdiv 2^{e_1-e_2}}{e_1}}) & \text{if $e_1 \ge e_2$} \\
\mathrm{round}_p({\float{s_1 \intdiv 2^{e_2-e_1} + s_2}{e_2}}) & \text{if $e_1 \le e_2$}
\end{cases} \\
\float{s_1}{e_1} \times \float{s_2}{e_2} &:= \mathrm{round}_p(\float{s_1s_2}{e_1+e_2}) \\
\float{s_1}{e_1} \div \float{s_2}{e_2} &:= 
\mathrm{round}_p({\float{s_1 \cdot 2^{p-1} \intdiv s_2}{e_1-e_2-p+1}}) \\
\float{s_1}{e_1} \le \float{s_2}{e_2} &\Leftrightarrow \begin{cases}
s_1 \le s_2 \intdiv 2^{e_1-e_2} & \text{if $e_1 \ge e_2$} \\
s_1 \intdiv 2^{e_2-e_1} \le s_2 & \text{if $e_1 \le e_2$.}
\end{cases}
\end{align*}
\end{definition}

The basic operations described above can be efficiently computed in parallel using simple hardware implementations in $\mathsf{TC}^0$ circuits, as formalized in the following lemma:

\begin{lemma}[Computing $\mathsf{FPN}$ operations with $\mathsf{TC}^0$ circuits,  Lemma 10 and Lemma
11 in~\cite{chi24}]\label{lem:fp_ops_tc0}
    We denote the number of digits as a positive integer $p$. If $p\leq \poly(n)$, then:
    \begin{itemize}
        \item Basic Operations: The ``$+$'', ``$\times$'', ``$\div$'', and comparison ($\leq$) of two $p$-bit $\mathsf{FPN}$s,  as defined in Definition~\ref{dfn:fp_num}, can be computed with $O(1)$-depth uniform threshold circuits with $\poly(n)$ size. Let the maximum circuit depth required for these basic operations be $d_{\mathrm{std}}$.
        \item Iterated Operations: The product of $n$ $p$-bit $\mathsf{FPN}$s and the sum of $n$ $p$-bit $\mathsf{FPN}$s (with rounding applied after summation) can be both computed with $O(1)$-depth uniform threshold circuits with $\poly(n)$ size. Let the maximum circuit depth required for multiplication be $d_\otimes$ and for addition be $d_\oplus$.  
    \end{itemize}
\end{lemma}

Beyond these basic floating-point operations, certain specialized floating-point operations are also known to be computable within $\mathsf{TC}^0$ circuits, as demonstrated in the following lemmas:   

\begin{lemma}[Computing $\exp$ with $\mathsf{TC}^0$ circuits, Lemma 12 in~\cite{chi24}]\label{lem:exp_tc0}
    Let $x\in\mathbb{F}_p$ be a $p$-bit $\mathsf{FPN}$. If $p\leq \poly(n)$, there exists an $O(1)$-depth uniform threshold circuit of size $\poly(n)$ that can compute $\exp(x)$ with relative error less than $2^{-p}$. Let the maximum circuit depth required for approximating $\exp$ be $d_{\mathrm{exp}}$.  
\end{lemma}

\begin{lemma}[Computing square root with $\mathsf{TC}^0$ circuits, Lemma 12 in~\cite{chi24}]\label{lem:sqrt_tc0}
Let $x\in\mathbb{F}_p$ be a $p$-bit $\mathsf{FPN}$. If $p\leq \poly(n)$, there exists an $O(1)$-depth uniform threshold circuit of $\poly(n)$ size that can compute $\sqrt{x}$ with relative error less than $2^{-p}$. Let the maximum circuit depth required for approximating square roots be $d_{\mathrm{sqrt}}$.  
\end{lemma}

\begin{lemma}[Computing matrix multiplication with $\mathsf{TC}^0$ circuits, Lemma 4.2 in~\cite{cll+24}]\label{lem:mat_prod_tc0}
Let two matrix operands be $A\in\mathbb{F}_p^{n_1\times n_2}, B\in \mathbb{F}_p^{n_2\times n_3}$. If $p\leq\poly(n)$ and $n_1, n_2, n_3\leq n$, then there exists an  $O(1)$-depth uniform threshold circuit of $\poly(n)$ size and depth $(d_{\mathrm{std}} + d_{\oplus})$ that can compute the matrix product $AB$.
\end{lemma}

\subsection{Graph Neural Networks}\label{sec:pre_gnn}

With the foundational framework of $\mathsf{FPN}$ operations, we now formalize the components of Graph Neural Networks (GNNs) in floating-point representations. This subsection commences by introducing activation functions and the softmax operation, which are fundamental building tools for GNN layers.  

\begin{definition}[ReLU]\label{dfn:relu}  
    For an embedding matrix $X \in \mathbb{F}_p^{n \times d}$, the output of the $\mathsf{ReLU}$ activation function is a matrix $Y \in \mathbb{F}_p^{n \times d}$, where each element is defined as $Y_{i,j} := \max\{0, X_{i,j}\}$ for all $1 \leq i \leq n$ and $1 \leq j \leq d$.  
\end{definition}  

\begin{definition}[Leaky ReLU]\label{dfn:leaky_relu}  
    For an embedding matrix $X \in \mathbb{F}_p^{n \times d}$ and a predefined negative slope $s \in \mathbb{F}_p$, the output of the $\mathsf{LeakyReLU}$ function is a matrix $Y \in \mathbb{F}_p^{n \times d}$, where each element is defined as $Y_{i,j} := \max\{0, X_{i,j}\} + s \cdot \min\{0, X_{i,j}\}$ for all $1 \leq i \leq n$ and $1 \leq j \leq d$.  
\end{definition}  

\begin{definition}[Softmax]\label{dfn:softmax}  
    Let $X \in \mathbb{F}_p^{n \times d}$ be an embedding matrix. The row-wise softmax function is defined element-wise as:  
    $$\mathsf{Softmax}(X)_{i,j} := \frac{\exp(X_{i,j})}{\sum_{k=1}^d \exp(X_{i,k})},$$  
    or equivalently in matrix form as:  
    $$\mathsf{Softmax}(X) := \diag(\exp(X) \cdot \boldsymbol{1}_d)^{-1} \exp(X),$$  
    where $\boldsymbol{1}_d$ is a $d$-dimensional column vector of ones.  
\end{definition}

These activation functions and softmax operations form the basis of GNN computation. We now introduce the convolution matrices central to the message-passing scheme of GNNs, focusing on three widely used GNN models: GCN~\cite{kw17}, GIN~\cite{xhl18}, and GAT~\cite{vcc18}.

\begin{definition}[GCN convolution matrix]\label{dfn:gcn_cmat}  
    Let $A \in \mathbb{F}_p^{n \times n}$ be the adjacency matrix of a graph with $n$ nodes, and let $D \in \mathbb{F}_p^{n \times n}$ be the diagonal degree matrix, where $D_{i,i} = \sum_{j=1}^n A_{i,j}$. Let $I \in \mathbb{F}_p^{n \times n}$ denote the identity matrix. The GCN convolution matrix is defined as:  
    $$C_{\mathrm{GCN}} := (D + I)^{-1/2}(A + I)(D + I)^{-1/2} \in \mathbb{F}_p^{n \times n}.$$  
\end{definition}  

\begin{definition}[GIN convolution matrix]\label{dfn:gin_cmat}  
    Let $A \in \mathbb{F}_p^{n \times n}$ be the adjacency matrix, and let $\epsilon \in \mathbb{F}_p$ be a constant. The GIN convolution matrix is defined as:  
    $$C_{\mathrm{GIN}} := A + (1 + \epsilon)I \in \mathbb{F}_p^{n \times n}.$$  
\end{definition}  

\begin{definition}[GAT convolution matrix]\label{dfn:gat_cmat}  
    Let $X \in \mathbb{F}_p^{n \times d}$ be the embedding matrix, and let $W \in \mathbb{F}_p^{d \times d}$ and $a \in \mathbb{F}_p^{2d}$ denote model weights. The attention weight matrix $E \in \mathbb{F}_p^{n \times n}$ is defined as:  
    $$E_{i,j} := \left\{  
        \begin{aligned}  
            &a^\top \mathsf{LeakyReLU}(WX_i || WX_j), & A_{i,j} = 1, \\  
            &-\infty, & \text{otherwise}.  
        \end{aligned}  
    \right.$$  
    The GAT convolution matrix is then given by:  
    $$C_{\mathrm{GAT}} := \mathsf{Softmax}(E).$$  
\end{definition}

Therefore, we unify the three commonly used graph convolution matrices with basic components to define a general GNN layer. 

\begin{definition}[One GNN layer]\label{dfn:gnn_layer}  
    Let $X \in \mathbb{F}_p^{n \times d}$ be the embedding matrix, $W \in \mathbb{F}_p^{d \times d}$ be the model weights, and $C \in \mathbb{F}_p^{n \times n}$ an arbitrary convolution matrix (e.g., $C_{\mathrm{GCN}}, C_{\mathrm{GIN}}, C_{\mathrm{GAT}}$). A single GNN layer is defined as:  
    $$\mathsf{GNN}_i(X) := \mathsf{ReLU}(CXW).$$  
\end{definition}  

By stacking multiple GNN layers, we obtain a multi-layer GNN capable of learning expressive node embeddings. Different prediction tasks, such as node-level, link-level, or graph-level tasks, require graph pooling operations to aggregate information from specific node subsets. We introduce two commonly used $\mathsf{READOUT}$ functions~\cite{bl23} for graph pooling:  

\begin{definition}[Graph average readout layer]\label{dfn:avg_pool}  
    Let $X \in \mathbb{F}_p^{n \times d}$ be a node embedding matrix. A graph average readout layer $\mathsf{READOUT}: \mathbb{F}_p^{n \times d} \rightarrow \mathbb{F}_p^d$ selects a subset $B = \{i_1, i_2, \dots, i_{|B|}\} \subseteq [n]$ and computes:  
    $$\mathsf{READOUT}(X) := \frac{1}{|B|}(X_{i_1} + X_{i_2} + \dots + X_{i_{|B|}}).$$  
\end{definition}  

\begin{definition}[Graph maximum readout layer]\label{dfn:max_pool}  
    Let $X \in \mathbb{F}_p^{n \times d}$ be a node embedding matrix. A graph maximum readout layer $\mathsf{READOUT}: \mathbb{F}_p^{n \times d} \rightarrow \mathbb{F}_p^d$ selects a subset $B = \{i_1, i_2, \dots, i_{|B|}\} \subseteq [n]$ and computes each dimension $j \in [d]$ as:  
    $$\mathsf{READOUT}(X)_j := \max\{X_{i_1,j}, X_{i_2,j}, \cdots, X_{i_{|B|},j}\}.$$  
\end{definition}  

\begin{remark}\label{rmk:pool_tasks}  
Graph readout functions in Definitions~\ref{dfn:avg_pool} and~\ref{dfn:max_pool} support decision problems at various levels, including but not limited to node, link, and graph tasks, since one can target specific nodes, edges, or the entire graph by appropriately selecting the subset $B$.
\end{remark}  

Finally, we introduce the MLP prediction head, essential for converting the aggregated embedding into specific predictions:  

\begin{definition}[MLP prediction head]\label{dfn:head}  
    Let $x \in \mathbb{F}_p^d$ be a single output embedding. With model weights $W \in \mathbb{F}_p^{d \times d}$, $w \in \mathbb{F}_p^d$, and bias $b \in \mathbb{F}_p^d$, the MLP prediction head is defined as:  
    $$\mathsf{Head}(x) := w^\top \mathsf{ReLU}(Wx + b).$$  
\end{definition}  

Finally, we integrate all the previously defined GNN components to present the complete formulation of a multi-layer GNN.

\begin{definition}[Multi-layer GNN]\label{dfn:gnn_multi_layer}
Let $m$ be the GNN layer number. Let $X \in \mathbb{F}_p^{n \times d}$ denote the input feature matrix. For each $i \in {1, \dots, m}$, let $\mathsf{GNN}_i$ represent the $i$-th GNN layer in Definition~\ref{dfn:gnn_layer}. Let $\mathsf{READOUT}$ be a graph readout function (Definition~\ref{dfn:avg_pool} or Definition~\ref{dfn:max_pool}), and $\mathsf{Head}$ be the MLP prediction head from Definition~\ref{dfn:head}. The $m$-layer GNN $\mathsf{GNN}: \mathbb{F}_p^{n \times d} \rightarrow \mathbb{F}_p$ is then defined as:
    $$\mathsf{GNN}(X) := \mathsf{Head}\circ\mathsf{READOUT}\circ\mathsf{GNN}_m\circ \mathsf{GNN}_{m-1} \circ \cdots \circ \mathsf{GNN}_1(X).$$
\end{definition}

\subsection{Circuit Complexity Classes}\label{sec:pre_comp_class}

In this subsection, we introduce the fundamental concepts of circuit complexity, a key concept of theoretical computer science and computational complexity.

\begin{definition}[Boolean circuit, Definition 6.1 in~\cite{ab09}]\label{dfn:bool_circ}
    A Boolean circuit with $n$ binary inputs and one binary output is a mapping $C_n$ between $\{0, 1\}^n$ and $\{0, 1\}$, represented as a directed acyclic graph (DAG). The graph consists of:
    \begin{itemize}
        \item $n$ input nodes, each with in-degree zero, corresponding to the input variables.
        \item One output node, each with out-degree zero, representing the output variable.
        \item Intermediate nodes, called gates, perform logical operations (e.g., $\mathsf{NOT}$, $\mathsf{OR}$, $\mathsf{AND}$) on the inputs. Each gate has one out-edge, representing the result of the computation. The in-degree of gate nodes is also referred to as their fan-in.
    \end{itemize}
    
    The structure of the graph allows the Boolean circuit to evaluate logical functions based on the input values, producing a corresponding output. 
\end{definition}

\begin{definition}[Complexity measures of Boolean circuits]\label{dfn:circ_measure}
    The \textbf{size} of a Boolean circuit $C$ is defined as the number of nodes in its computation graph. The \textbf{depth} of $C$ is the length of the longest path in its computation graph.
\end{definition}

To analyze the expressiveness of specific Boolean circuits, we first formally define the concept of languages recognized by a circuit family.

\begin{definition}[Language recognition with circuit families, Definition 2 in~\cite{ms23}]\label{dfn:circ_recog}
    A circuit family $\mathcal{C}$ denotes a set of Boolean circuits. The circuit family $\mathcal{C}$ can recognize a language  $L \subseteq \{0, 1\}^*$ if, for every string $s \in \{0, 1\}^*$, there exists a Boolean circuit $C_{|s|} \in \mathcal{C}$ with input size $|s|$ such that $C_{|s|}(s) = 1$ if and only if $s \in L$.
\end{definition}

We now define complexity classes of languages based on the circuit families capable of recognizing them, with a focus on the resources (e.g., depth, size) these circuits require:

\begin{definition}[$\mathsf{NC}^i$, Definition 6.21 in~\cite{ab09}]\label{dfn:class_nci}
    A language $L$ belongs to the class $\mathsf{NC}^i$ (Nick's Class) if there is a family of Boolean circuits that can recognize $L$, where the circuits have $\poly(n)$ size, $O((\log n)^i)$ depth, and $\mathsf{NOT}$, $\mathsf{OR}$, $\mathsf{AND}$ logical gates with bounded fan-in.
\end{definition}

\begin{definition}[$\mathsf{AC}^i$, Definition 6.22 in~\cite{ab09}]
    A language $L$ belongs to the class $\mathsf{AC}^i$ if there is a family of Boolean circuits that can recognize $L$, where the circuits have $\poly(n)$ size, $O((\log n)^i)$ depth, and $\mathsf{NOT}$, $\mathsf{OR}$, $\mathsf{AND}$ logical gates with unbounded fan-in. 
\end{definition}

\begin{definition}[$\mathsf{TC}^i$, Definition 4.34 in~\cite{vol99}]\label{dfn:class_tci}
    A language $L$ belongs to the class $\mathsf{TC}^i$ if there is a family of Boolean circuits that can recognize $L$, where the circuits have $\poly(n)$ size, $O((\log n)^i)$ depth, and unbounded fan-in gates for $\mathsf{NOT}$, $\mathsf{OR}$, $\mathsf{AND}$, and $\mathsf{MAJORITY}$ operations, where a $\mathsf{MAJORITY}$ gate outputs one if more than $1/2$ of its inputs are ones.
\end{definition}

\begin{remark}\label{rmk:tci_threshold}
    The $\mathsf{MAJORITY}$ gates in Definition~\ref{dfn:class_tci} can be substituted with prime-modulus $\mathsf{MOD}$ gates or $\mathsf{THRESHOLD}$ gates. Any Boolean circuit utilizing one of these gates is called a threshold circuit.
\end{remark}

Next, we define the complexity class $\mathsf{DET}$, which plays a critical role in certain hardness results.
\begin{definition}[$\mathsf{DET}$, on page 12 of~\cite{coo85}]\label{dfn:class_det}
    A Boolean circuit belongs to the $\mathsf{DET}$ family if it computes a problem that is $\mathsf{NC}^1$-reducible to the computation of the determinant $\det(A)$ of an $n \times n$ matrix $A$ with $n$-bit integers.
\end{definition}

\begin{fact}[Inclusion of circuit complexity classes, page 110 on~\cite{ab09}, Corollary 4.35 of~\cite{vol99}, page 19 of~\cite{coo85}]\label{fact:complexity_relation}
    We have $\mathsf{NC}^i\subseteq\mathsf{AC}^i\subseteq\mathsf{TC}^i\subseteq\mathsf{NC}^{i+1}$ for all $i\in\mathbb{N}$. Specifically, for $i=1$, we have $\mathsf{NC}^1\subseteq\mathsf{DET}\subseteq\mathsf{NC}^2$. 
\end{fact}

\begin{remark}\label{rmk:tc0_nc1_open}
    As noted on page 116 of~\cite{vol99} and page 110 of~\cite{ab09}, it is well-known that $\mathsf{NC}^i \subsetneq \mathsf{AC}^i \subsetneq \mathsf{TC}^i$ for $i=0$. However, whether $\mathsf{TC}^0 \subsetneq \mathsf{NC}^1$ remains an open problem in circuit complexity. Additionally, as discussed on page 18 of~\cite{coo85}, it is unlikely that $\mathsf{DET} \subseteq \mathsf{AC}^1$, despite $\mathsf{DET}$ being bounded between $\mathsf{NC}^1$ and $\mathsf{NC}^2$.
\end{remark}

We have formulated non-uniform circuit families that allow different structures for different input lengths. While flexible, this lack of uniformity is impractical compared to computational models like Turing machines, where the same device handles all input lengths. To address this, we introduce uniform circuit families, where circuits for all input lengths can be systematically generated by a Turing machine under specific time and space constraints. We begin by introducing $\mathsf{L}$-uniformity.

\begin{definition}[$\mathsf{L}$-uniformity, Definition 6.5 in~\cite{ab09}]\label{dfn:l_uniform}
    Let $\mathcal{C}$ denote a circuit family, and let $\mathsf{C}$ denote a language class recognizable by $\mathcal{C}$. A language $L \subseteq \{0, 1\}^*$ belongs to the $\mathsf{L}$-uniform class of $\mathsf{C}$ if there exists an $O(\log n)$-space Turing machine that can produce a circuit $C_n \in \mathcal{C}$ with $n$ variables for any input $1^n$. The circuit $C_n$ must recognize $L$ for inputs of size $n$.  
\end{definition}  

Next, we define $\mathsf{DLOGTIME}$-uniformity, which refines $\mathsf{L}$-uniformity by introducing a more computationally practical time constraint. Throughout this paper, references to uniform circuit families specifically denote their $\mathsf{DLOGTIME}$-uniform versions.

\begin{definition}[$\mathsf{DLOGTIME}$-uniformity, Definition 4.28 in~\cite{bi94}]  
    Let $\mathsf{C}$ be an $\mathsf{L}$-uniform language class as defined in Definition~\ref{dfn:l_uniform}. A language $L \subseteq \{0, 1\}^*$ belongs to the $\mathsf{DLOGTIME}$-uniform class of $\mathsf{C}$ if there exists a Turing machine that can produce a circuit $C_n \in \mathcal{C}$ with $n$ variables for any input $1^n$ within $O(\log n)$ time. The circuit $C_n$ must recognize the language $L$ for inputs of size $n$.  
\end{definition}  
\section{Complexity of Graph Neural Networks}\label{sec:complexity}

In this section, we establish foundational complexity results for each component of a graph neural network (GNN) and then combine these results to derive a circuit complexity bound for the entire multi-layer GNN. Section~\ref{sec:comp_activation} examines activation functions, which form the basics for graph convolution computations. Section~\ref{sec:comp_gconv} explores the computation of graph convolution matrices. Section~\ref{sec:comp_layer} analyzes a single GNN layer, the fundamental building block of a multi-layer GNN. In Section~\ref{sec:comp_head}, we investigate graph readout functions and the MLP prediction head, essential for making predictions with a multi-layer GNN. Finally, in Section~\ref{sec:comp_gnn}, we integrate all components and analyze the complete multi-layer GNN structure, culminating in Section~\ref{sec:comp_result}, where we present the key result: the circuit complexity bound of graph neural networks.

\subsection{Computing Activation Functions}\label{sec:comp_activation}

In this subsection, we first establish a useful fact about computing pairwise $\max$ and $\min$ functions with $\mathsf{TC}^0$ circuits. We then demonstrate that the $\mathsf{ReLU}$ and $\mathsf{LeakyReLU}$ activation functions on embedding matrices can be efficiently computed by uniform threshold circuits. 

\begin{fact}[Computing pairwise $\max$ and $\min$ with $\mathsf{TC}^0$ circuits]\label{fact:max_min}
    Let $a, b\in\mathbb{F}_{p}$ be two $\mathsf{FPN}$s. If $p\leq\poly(n)$, there is an $O(1)$-depth uniform threshold circuit of $\poly(n)$ size and $(d_{\mathrm{std}} + 3)$ depth that can compute $\min\{a, b\}$ and $\max\{a, b\}$. 
\end{fact}
\begin{proof} 
    For two $\mathsf{FPN}$s $a$ and $b$, we first compute the comparison result $c=1$ f $a\leq b$, and $c=0$ otherwise, leveraging an $O(1)$-depth uniform threshold circuit with depth $d_{\mathrm{std}}$, as stated in Lemma~\ref{dfn:fp_ops}. Once $c$ is computed, each bit $i\in[p]$ of $\min\{a, b\}$ and $\max\{a, b\}$ can be determined as $\min\{a, b\}_i = (c \land a_i) \lor (\lnot c \land b_i)$ and $ \max\{a, b\}_i = (c \land b_i) \lor (\lnot c \land a_i)$ respectively, which can be computed with a $3$-depth Boolean circuit in parallel. Combining the depths for comparison and logical operations, the overall circuit depth is $d_{\mathrm{std}} + 3$. Since $p\leq \poly(n)$ and each operation is computable with a polynomial-size circuit, we conclude that $\poly(n)$ is the size of the entire circuit. This completes the proof.
\end{proof}

\begin{lemma}[Computing $\mathsf{ReLU}$ with $\mathsf{TC}^0$ circuits]\label{lem:relu_tc0}
    Let $X \in \mathbb{F}_p^{n\times d}$ be a matrix. If $p\leq \poly(n), d= O(n)$, there is aa constant-depth uniform threshold circuit of $\poly(n)$ size and depth $(d_{\mathrm{std}} + 3)$ that can compute $\mathsf{ReLU}(X)$.
\end{lemma}
\begin{proof}
    For each pair of subscript $i, j\in[n]$, the entry $\mathsf{ReLU}(X)_{i, j}$ is formulated as $\mathsf{ReLU}(X)_{i, j} = \max\{0, X_{i,j}\}$ following Definition~\ref{dfn:relu}. By Fact~\ref{fact:max_min}, the computation of $\max$ function can be finished with a uniform threshold circuit with depth $(d_{\mathrm{std}} + 3$) and $\poly(n)$ size. If we compute all the $O(nd) \leq O(n^2) \leq \poly(n)$ entries in parallel, we can also compute $\mathsf{ReLU}(X)$ with depth $(d_{\mathrm{std}} + 3)$ and size $\poly(n)$. Hence, we finish the proof. 
\end{proof}

\begin{lemma}[Computing $\mathsf{LeakyReLU}$ with $\mathsf{TC}^0$ circuits]\label{lem:activ_tc0}
    Let $X \in \mathbb{F}_p^{n\times d}$ be a matrix and $s\in\mathbb{F}_p$ is the negative slope. If $p\leq \poly(n), d= O(n)$, there is an $O(1)$-depth uniform threshold circuit of $\poly(n)$ size and depth $(3d_{\mathrm{std}} + 3)$ that can compute $\mathsf{LeakyReLU}(X)$. 
\end{lemma}
\begin{proof}
    Considering $\forall i, j\in[n]$, the entry $\mathsf{LeakyReLU}(X)_{i, j}$ is formulated as $\mathsf{LeakyReLU}(X)_{i, j} = \max\{0, X_{i,j}\} + s\cdot \min\{0, X_{i,j}\}$ following Definition~\ref{dfn:relu}. By Fact~\ref{fact:max_min}, both $ \max\{0, X_{i,j}\}$ and $\min\{0, X_{i,j}\}$ function can be computed with a uniform threshold circuit with depth $d_{\mathrm{std}}+3$ and $\poly(n)$ size. After that, we can multiply $\min\{0, X_{i,j}\}$ with slope $s$ and further add it with $\max\{0, X_{i,j}\}$, which can be computed with a $2d_{\mathrm{std}}$-depth polynomial size circuit by applying Lemma~\ref{lem:fp_ops_tc0} twice. Combining the above computation steps, we have $d_{\mathrm{total}} = d_{\mathrm{std}}+3 + 2d_{\mathrm{std}} = 3d_{\mathrm{std}} + 3.$ Since there are $O(nd) \leq O(n^2) \leq poly(n)$ entries in $X$ and all the computations are finished with polynomial-size circuits, the size of the entire circuit is also $\poly(n)$. Thus, we complete the proof.
\end{proof}

\subsection{Computing Graph Convolution Matrices}\label{sec:comp_gconv}

In this subsection, we present results on the computation of representative graph convolution matrices using uniform threshold circuits.

\begin{lemma}[Computing GCN convolution matrix with $\mathsf{TC}^0$ circuits]\label{lem:gcn_conv_tc0}
    If $p\leq\poly(n)$, then there is a $\poly(n)$ size uniform threshold circuit with depth $(3d_\oplus + 3d_{\mathrm{std}} + d_{\mathrm{sqrt}})$, which can compute the graph convolution matrix $C_{\mathrm{GCN}}$ in Definition~\ref{dfn:gcn_cmat}.
\end{lemma}
\begin{proof}
    Considering all the rows $\forall i\in[n]$ in the adjacency matrix $A$, the corresponding element of the degree matrix $D$ is written as $D_{i, j}=\sum_{j=1}^n A_{i,j}$ by definition. Following Lemma~\ref{lem:fp_ops_tc0}, each iterated summation is computable with a $\poly(n)$ size uniform threshold circuit with $d_\oplus$ depth. After that, we can compute all the elements of $(D+I)$ and $(A+I)$ in parallel, which can be finished with a uniform threshold circuit having $\poly(n)$ size and $d_{\mathrm{std}}$ depth by Lemma~\ref{dfn:fp_ops}. Then we can compute all the elements in $(D+I)^{1/2}$ with a $d_{\mathrm{sqrt}}$-depth circuit following Lemma~\ref{lem:sqrt_tc0}.  
    Finally, we compute the matrix multiplication $(D+I)^{1/2}(A+I)(D+I)^{1/2}$ with uniform threshold circuit of $\poly(n)$ size and depth $2(d_{\mathrm{std}} + d_\oplus)$ by applying Lemma~\ref{lem:mat_prod_tc0} twice. Combining all the aforementioned circuits, we have
    \begin{align*}
        d_{\mathrm{total}} =&~ d_\oplus + d_{\mathrm{std}} + d_{\mathrm{sqrt}} + 2(d_{\mathrm{std}} + d_\oplus)\\ 
        =&~ 3d_\oplus + 3d_{\mathrm{std}} + d_{\mathrm{sqrt}}.
    \end{align*}
Since each step utilizes a circuit of size $\poly(n)$, the size of the entire circuit is also $\poly(n)$. Therefore, we can conclude that a $\poly(n)$ size and $(3d_\oplus + 3d_{\mathrm{std}} + d_{\mathrm{sqrt}})$ depth uniform threshold circuit can compute the GCN convolution matrix $C_{\mathrm{GCN}}$.
\end{proof}

\begin{lemma}[Computing GIN convolution matrix with $\mathsf{TC}^0$ circuits]\label{lem:gin_conv_tc0}
    If $p\leq\poly(n)$, then there is a $\poly(n)$ size uniform threshold circuit with depth $3d_{\mathrm{std}}$, which can compute the graph convolution matrix $C_{\mathrm{GIN}}$ in Definition~\ref{dfn:gin_cmat}.
\end{lemma}
\begin{proof}
    By Lemma~\ref{lem:fp_ops_tc0}, we can first simply compute the constant $(1+ \epsilon)$ with a $d_{\mathrm{std}}$-depth uniform threshold circuit with size $\poly(n)$. Then we can consider each element $i, j\in[n]$ of the adjacency matrix $A$, computing $(C_{\mathrm{GIN}})_{i, j} = A_{i, j} + (1+ \epsilon) \cdot I_{i, j}$ in parallel. Applying Lemma~\ref{lem:fp_ops_tc0} twice, we can conclude that all the matrix elements can be produced with a $\poly(n)$ size uniform threshold circuit with depth $2 d_{\mathrm{std}}$. Thus, by combining all the aforementioned circuits, we have
    $d_{\mathrm{total}} =d_{\mathrm{std}} + 2d_{\mathrm{std}} = 3d_{\mathrm{std}}.$ Therefore, we can conclude that a $\poly(n)$ size and $3d_{\mathrm{std}}$ depth uniform threshold circuit can compute the GIN convolution matrix $C_{\mathrm{GIN}}$.
\end{proof}

Before discussing the $\mathsf{TC}^0$ implementation of the GAT convolution matrix, we introduce a key fact on the softmax mechanism, which is crucial for attention computation in GAT.

\begin{lemma}[Computing $\mathsf{Softmax}$ with $\mathsf{TC}^0$ circuits]\label{lem:softmax_tc0}
    Let $X\in\mathbb{F}_p^{n\times d}$ be a matrix. If $p\leq \poly(n), d= O(n)$, there is a $\poly(n)$ size uniform threshold circuit with depth $(d_{\mathrm{exp}} + 3 d_{\mathrm{std}} + 2 d_\oplus + 1)$ that can compute $\mathsf{Softmax}(X)$ in Definition~\ref{dfn:softmax}. 
\end{lemma}
\begin{proof}
    By Lemma~\ref{lem:exp_tc0}, all the $O(nd) \leq O(n^2) \leq \poly(n)$ entries of $\exp(X)$ can be produced in parallel through a polynomial-size uniform threshold circuit with depth $d_{\mathrm{exp}}$. Next, we compute $\exp(X)\cdot \boldsymbol{1}_d$ using a circuit of depth $(d_{\mathrm{std}} + d_\oplus)$ as per Lemma~\ref{lem:mat_prod_tc0}, and extract the diagonal matrix $\diag(\exp(X)\cdot \boldsymbol{1}_d)$ using a depth-1 circuit. Subsequently, we compute the inverse of the diagonal matrix with a depth-$d_{\mathrm{std}}$ circuit by inverting all diagonal elements in parallel, as guaranteed by Lemma~\ref{lem:fp_ops_tc0}. We then multiply $\diag(\exp(X)\cdot \boldsymbol{1}_d)^{-1}$ and $\exp(X)$ using a polynomial-size circuit with depth $(d_{\mathrm{std}} + d_\oplus)$, again following Lemma~\ref{lem:mat_prod_tc0}. 

    Combining these circuits, the total depth is:
    \begin{align*}
        d_{\mathrm{total}} =&~ d_{\mathrm{exp}} + (d_{\mathrm{std}} + d_\oplus) + 1 + d_{\mathrm{std}} + (d_{\mathrm{std}} + d_\oplus) \\ 
        =&~ d_{\mathrm{exp}} + 3d_{\mathrm{std}} + 2d_\oplus + 1.
    \end{align*}
    Since the number of parallel operations is $\poly(n)$, $\mathsf{Softmax}(X)$ can be computed with a $\poly(n)$ size uniform threshold circuit with depth $(d_{\mathrm{exp}} + 3d_{\mathrm{std}} + 2d_\oplus + 1)$. This completes the proof.
\end{proof}

Building on the softmax computation fact, we show that the GAT convolution matrix can indeed be computed with $\mathsf{TC}^0$ circuits, as formalized in the following lemma.

\begin{lemma}[Computing GAT convolution matrix with $\mathsf{TC}^0$ circuits]\label{lem:gat_conv_tc0}
    If $p\leq\poly(n), d= O(n)$, then there is a $\poly(n)$ size uniform threshold circuit with depth $(d_{\mathrm{exp}} + 9d_{\mathrm{std}} + 4d_\oplus + 4)$, which can compute the graph convolution matrix $C_{\mathrm{GAT}}$ in Definition~\ref{dfn:gat_cmat}.
\end{lemma}
\begin{proof}
     Let $\omega \in \mathbb{F}_p$ denote a $\mathsf{FPN}$ with a negative infinity value, and let the model weight $a = (\alpha_1, \dots, \alpha_p, \alpha_{p+1}, \dots, \alpha_{2p}) \in \mathbb{F}_p^{2d}$. We define $a_1 = (\alpha_1, \dots, \alpha_p) \in \mathbb{F}_p^d$ and $a_2 = (\alpha_{p+1}, \dots, \alpha_{2p}) \in \mathbb{F}_p^d$, such that $a = a_1 \| a_2$. For all $i, j \in [n]$, each entry $E_{i,j}$ of the attention weight matrix is given by:
    \begin{align*}
        E_{i,j} =&~ A_{i,j} \cdot a^\top \mathsf{LeakyReLU}(WX_i \| WX_j) + (1 - A_{i,j}) \cdot \omega \\
        =&~ A_{i,j} \cdot a_1^\top \mathsf{LeakyReLU}(WX_i) + A_{i,j} \cdot a_2^\top \mathsf{LeakyReLU}(WX_j) + (1 - A_{i,j}) \cdot \omega,
    \end{align*}
    following Definition~\ref{dfn:gat_cmat}. We examine the three terms separately:
    \begin{align*}
        E_{i,j} = \underbrace{A_{i,j} \cdot a_1^\top \mathsf{LeakyReLU}(WX_i)}_{:=E^1_{i,j}} + \underbrace{A_{i,j} \cdot a_2^\top \mathsf{LeakyReLU}(WX_j)}_{:=E^2_{i,j}} + \underbrace{(1 - A_{i,j}) \cdot \omega}_{:=E^3_{i,j}}.
    \end{align*}
    
    For the first term $E^1_{i,j}$, there is a $\poly(n)$ size and $(d_{\mathrm{std}} + d_\oplus)$ uniform threshold circuit to compute $WX_i$ following Lemma~\ref{lem:mat_prod_tc0}. Then we can compute $\mathsf{LeakyReLU}(WX_i)$ with a $(3d_{\mathrm{std}} + 3)$ depth circuit by Lemma~\ref{dfn:leaky_relu}. After that, we can combine Lemma~\ref{lem:mat_prod_tc0} and Lemma~\ref{lem:fp_ops_tc0} to multiply $A_{i,j }$, $a^\top_1$ and $\mathsf{LeakyReLU}(WX_i)$ with a polynomial size and $(2d_{\mathrm{std}} + d_\oplus)$ depth threshold circuit. Combining all the circuits above, we can get the total depth $d_1$ for the first $E^1_{i,j}$ term as follows:
        \begin{align*}
            d_1 =&~ (d_{\mathrm{std}} + d_\oplus) + (3d_{\mathrm{std}} + 3) + (2d_{\mathrm{std}} + d_\oplus) \\
            =&~ 6d_{\mathrm{std}} + 2d_\oplus + 3.
        \end{align*} 
        
    For the second term $E^2_{i,j}$, since its computation is equivalent to the first term, we can conclude that the depth of the second term $d_2$ equals $d_1$. 

    For the third term $E^3_{i,j}$, we can simply apply Lemma~\ref{lem:fp_ops_tc0} twice and conclude that $(1 - A_{i,j})\cdot \omega$ can be produced with a uniform threshold circuit, in which the size is $\poly(n)$ and the depth is $d_3 = 2d_{\mathrm{std}}$.

    Integrating the computation of all the three terms $E^1_{i,j}, E^2_{i,j}, E^3_{i,j}$ for $E_{i, j}$, we can compute these three terms in parallel, and then obtain a polynomial-size uniform threshold circuit with depth $d_e = \max\{d_1, d_2, d_3\} = d_1 = 6d_{\mathrm{std}} + 2d_\oplus + 3$. Since we have obtained a uniform threshold circuit to compute each element in $E_{i,j}$ and the elements to be computed in parallel is of $O(nd) \leq O(n^2) \leq \poly(n)$, we can compute all the elements in $E_{i,j}$ in parallel without increasing the circuit's polynomial size and depth. At last, we can compute $\mathsf{Softmax}(E)$ to obtain the final convolution matrix, and this can be finished with a uniform threshold circuit with depth $(d_{\mathrm{exp}} + 3 d_{\mathrm{std}} + 2 d_\oplus + 1)$ following Lemma~\ref{lem:softmax_tc0}. Therefore, the final depth of the circuit is the summation of the circuit for computing the matrix $E$ and applying the $\mathsf{Softmax}$ operation, which is:
    \begin{align*}
        d_{\mathrm{total}} =&~ d_e + d_{\mathrm{exp}} + 3 d_{\mathrm{std}} + 2 d_\oplus + 1 \\
        =&~ 6d_{\mathrm{std}} + 2d_\oplus + 3 + d_{\mathrm{exp}} + 3 d_{\mathrm{std}} + 2 d_\oplus + 1  \\ 
        =&~ d_{\mathrm{exp}} + 9d_{\mathrm{std}} + 4d_\oplus + 4.
    \end{align*}
    Therefore, we can conclude that the GAT convolution matrix is computable with a uniform threshold circuit within polynomial size and $(d_{\mathrm{exp}} + 9d_{\mathrm{std}} + 4d_\oplus + 4)$ depth, which finishes the proof. 
\end{proof}

After establishing the computation of all graph convolution matrices, we introduce a simplified notation for their circuit depths. Specifically, we define the depth of the uniform threshold circuit used for graph convolution matrices as $d_{\mathrm{conv}} = \max\{3d_\oplus + 3d_{\mathrm{std}} + d_{\mathrm{sqrt}},3d_{\mathrm{std}},d_{\mathrm{exp}} + 9d_{\mathrm{std}} + 4d_\oplus + 4\}$, based on the results in Lemma~\ref{lem:gcn_conv_tc0}, Lemma~\ref{lem:gin_conv_tc0}, and Lemma~\ref{lem:gat_conv_tc0}.

\subsection{Computing Single Graph Neural Network Layer}\label{sec:comp_layer}
In this subsection, we combine the results on graph convolution matrices and basic activation functions to determine the circuit depth requirement for a complete GNN layer.

\begin{lemma}[Computing a single GNN layer with $\mathsf{TC}^0$ circuits]\label{lem:one_gnn_layer}
    If $p\leq \poly(n), d= O(n)$
    , then the GNN layer in Definition~\ref{dfn:gnn_layer} can be computed with a uniform threshold circuit
of $\poly(n)$ size and depth $(d_{\mathrm{conv}} + 3d_{\mathrm{std}} + 2d_{\oplus} + 3)$. 
\end{lemma}
\begin{proof}
By Lemma~\ref{lem:gcn_conv_tc0}, Lemma~\ref{lem:gin_conv_tc0} and Lemma~\ref{lem:gat_conv_tc0}, we can obtain that the graph convolution matrix can be computed with a $d_{\mathrm{conv}}$ depth polynomial size threshold circuit. By Lemma~\ref{lem:mat_prod_tc0}, we can further conclude that $CXW$ can be produced by a circuit with polynomial size and $2(d_{\mathrm{std}} + d_{\oplus})$ depth. Then we can compute $\mathsf{ReLU}(CXW)$ with a $(d_{\mathrm{std}} + 3)$-depth and polynomial size circuit following Lemma~\ref{lem:relu_tc0}. Combining all the circuit depths mentioned above, we have:
\begin{align*}
    d_{\mathrm{total}} &=~ d_{\mathrm{conv}} + 2(d_{\mathrm{std}} + d_{\oplus}) + (d_{\mathrm{std}} + 3) \\
    &=~ d_{\mathrm{conv}} + 3d_{\mathrm{std}} + 2d_{\oplus} + 3.   
\end{align*}

Since all the steps use a polynomial step circuit, we can conclude that a single GNN layer can be computed with a uniform threshold circuit having $\poly(n)$ size and $(d_{\mathrm{conv}} + 3d_{\mathrm{std}} + 2d_{\oplus} + 3)$ depth. Thus, we finish the proof.
\end{proof}

\subsection{Computing Pooling Layer and Prediction Head}\label{sec:comp_head}

As shown in Definition~\ref{dfn:avg_pool}, Definition~\ref{dfn:max_pool}, and Definition~\ref{dfn:head}, we have described basic components for transforming GNN node embeddings into predictions. Here, we present results on computing these components with uniform threshold circuits, starting with two graph readout layers. 

\begin{lemma}[Computing graph average readout with $\mathsf{TC}^0$ circuits]\label{lem:pool_tc0}
    If $p\leq \poly(n), d= O(n)$, then the graph average readout layer in Definition~\ref{dfn:avg_pool} can be computed with a $\poly(n)$ size uniform threshold circuit with depth $(3d_{\mathrm{std}} + 2d_\oplus)$. 
\end{lemma}
\begin{proof}
    Let $\beta = (\beta_1, \beta_2,\dots, \beta_n)^\top\in\mathbb{F}_p^n$ be the vectorized form of index set $B$, where $\beta_{i_k}=1$ for all $i_k\in B$ and $\beta_{i_k}=0$ otherwise. Therefore, the output $\mathsf{READOUT}(X)$ in Definition~\ref{dfn:avg_pool} is equivalent to the following:
    \begin{align*}
        \mathsf{READOUT}(X) =&~ \frac{1}{|B|} \underbrace{X^\top}_{d\times n} \cdot \underbrace{\diag(\beta)}_{n\times n} \cdot \underbrace{\boldsymbol{1}_n}_{n\times 1},
    \end{align*}
    where $\boldsymbol{1}_n$ is a $n$-dimensional one vector. Thus, we can apply Lemma~\ref{lem:mat_prod_tc0} twice and obtain a $\poly(n)$ size and $2(d_{\mathrm{std}} + d_\oplus)$ depth uniform threshold circuit that can compute matrix product $X^\top\cdot\diag(\beta)\cdot\boldsymbol{1}_n$. Thus, we follow Lemma~\ref{lem:fp_ops_tc0} to multiply $\frac{1}{|B|}$ with a depth $d_{\mathrm{std}}$ circuit at $\poly(n)$ size. Combining the two circuits mentioned above, we have the total depth:
    $$d_{\mathrm{total}} = 2(d_{\mathrm{std}} + d_\oplus) + d_{\mathrm{std}} = 3d_{\mathrm{std}} + 2d_\oplus. $$
Since all the partial circuits are in polynomial size, we can conclude that there is a $\poly(n)$ size and $(3d_{\mathrm{std}} + 2d_\oplus)$ depth uniform threshold circuit which can compute the graph pooling layer. Hence, we finish the proof. 
\end{proof}

Before delving into the graph maximum readout layer, we establish a useful fact about computing maximum and minimum values in $\mathsf{TC}^0$. 

\begin{fact}[Computing $\max$ and $\min$ of $n$ values with $\mathsf{TC}^0$ circuits]\label{fact:max_min_n_tc0}
    Let $x_1, x_2, \cdots, x_n\in\mathbb{F}_p$ be $n$ $\mathsf{FPN}$s. If precision $p\leq \poly(n)$, there is an $O(1)$-depth uniform threshold circuit of size $\poly(n)$ and depth $(d_{\mathrm{std}} + 3)$ that can compute the maximum and minimum of these $\mathsf{FPN}$s, i.e., $\max\{x_1, \cdots, x_n\}$ and $\min\{x_1, \cdots, x_n\}$.
\end{fact}

\begin{proof}
For all the $\mathsf{FPN}$ pairs $i, j \in [n]$, we first compute the comparison results between all pairs of the $n$ numbers. Let $C^{\max}\in\{0, 1\}^{n\times n}$, where $C^{\max}_{i,j} = 1$ if $x_i \geq x_j$ and $C^{\max}_{i,j}=0$ otherwise, and $C^{\min}\in\{0, 1\}^{n\times n}$, where $C^{\min}_{i,j} = 1$ if $x_i \leq x_j$ and $C^{\min}_{i,j}=0$ otherwise. Since there are in total $O(n^2) \leq \poly(n)$ comparison operations, we conclude that each $C_{i,j}$ can be produced in parallel using a $\poly(n)$ size uniform threshold circuit with $d_{\mathrm{std}}$ depth, following Lemma~\ref{lem:fp_ops_tc0}. 

Once $C^{\max}$ and $C^{\min}$ are computed, we compute the dominance vector $D^{\max} \in \{0, 1\}^n$ in parallel using a 1-depth Boolean circuit, where $D^{\max}_i = \bigwedge_{j=1}^n C^{\max}_{i,j}$. Here, $D^{\max}_i = 1$ indicates that $x_i$ is a maximum value. Similarly, we can compute $D^{\min} \in \{0, 1\}^n$ in parallel, where $D^{\min}_i = \bigwedge_{j=1}^n C^{\min}_{i,j}$, to identify the minimum value.

To select the maximum value, we compute each bit $k \in [p]$ of the output $o^{\max} = \max\{x_1, x_2, \cdots, x_n\}$ in parallel using a 2-depth Boolean circuit:
$$o_k^{\max} = \bigvee_{i=1}^n (D^{\max}_i \land (x_i)_k),$$
in which $(x_i)_k$ is the $k$-th bit of $x_i$. Since all maximum values are equal, if a specific bit $k$ of all maximum values is 1, the OR operation ensures \(o_k^{\max} = 1\); similarly, if all maximum values have 0 in bit $k$, then \(o_k^{\max} = 0\). Thus, this approach guarantees correctness regardless of the number of maximum values. Similarly, each bit $k \in [p]$ of the minimum output $o^{\min} = \min\{x_1, x_2, \cdots, x_n\}$ is computed using:
$$o_k^{\min} = \bigvee_{i=1}^n (D^{\min}_i \land (x_i)_k).$$

Combining all the steps above, the total circuit depth is therefore $d_{\text{total}} = d_{\mathrm{std}} + 1 + 2 = d_{\mathrm{std}} + 3$. Since $p \leq \poly(n)$ and all the steps are computed by a circuit of $\poly(n)$ size, we conclude that the total circuit size is also $\poly(n)$. This completes the proof.
\end{proof}

With this fact, the graph maximum readout layer can be computed seamlessly using $\mathsf{TC}^0$ circuits, as stated in the following lemma:

\begin{lemma}[Computing graph maximum readout with $\mathsf{TC}^0$ circuits]\label{lem:maxpool_tc0}  
    If $p \leq \poly(n), d= O(n)$, the graph maximum readout layer in Definition~\ref{dfn:max_pool} can be computed using a $\poly(n)$ size uniform threshold circuit with depth $(d_{\mathrm{std}} + 3)$.  
\end{lemma}  
\begin{proof}  
For each entry $i \in [d]$ of the output $\mathsf{READOUT}(X)$, since it corresponds to the maximum of $|B| \leq n$ numbers, we can compute them in parallel through a $\poly(n)$ size uniform threshold circuit with depth $(d_{\mathrm{std}} + 3)$ by Fact~\ref{fact:max_min_n_tc0}. This completes the proof.  
\end{proof}  

Following the computation of graph readout functions, we introduce a simplified notation for their circuit depths. Specifically, we define the depth of the uniform threshold circuit used for computing the graph readout layer as $d_{\mathrm{read}} = \max\{3d_{\mathrm{std}} + 2d_\oplus, d_{\mathrm{std}} + 3\}$, which follows the results in Lemma~\ref{lem:pool_tc0} and Lemma~\ref{lem:maxpool_tc0}.

Next, we introduce the circuit depth for computing an MLP prediction head.
\begin{lemma}[Computing MLP head with $\mathsf{TC}^0$ circuits]\label{lem:pred_head}
    If $p\leq \poly(n), d= O(n)$, then the MLP prediction head in Definition~\ref{dfn:head} can be computed with a $\poly(n)$ size uniform threshold circuit with depth $(4d_{\mathrm{std}} + 2d_\oplus + 3)$. 
\end{lemma}
\begin{proof}
By Lemma~\ref{lem:mat_prod_tc0}, matrix product $Wx$ is computable with a $\poly(n)$ size and $(d_{\mathrm{std}} + d_\oplus)$ depth uniform threshold circuit. After that, for each $i\in[d]$, we can compute each element $(Wx + b)_i = (Wx)_i + b_i$ in parallel with a $\poly(n)$ size and $d_{\mathrm{std}}$ size uniform threshold circuit by Lemma~\ref{lem:fp_ops_tc0}. Then, we can further apply the activation function with a polynomial size and $(d_{\mathrm{std}} + 3)$ uniform threshold circuit by applying Lemma~\ref{lem:relu_tc0}. Finally, following Lemma~\ref{lem:mat_prod_tc0}, the inner product between $w$ and $\mathsf{ReLU}(Wx+b)$ is computable with a polynomial size and $d_{\mathrm{std}} + d_\oplus$ uniform threshold circuit. 

Combining all the aforementioned circuits, we have the total circuit depth:
\begin{align*}
    d_{\mathrm{total}} =&~ (d_{\mathrm{std}} + d_\oplus) + d_{\mathrm{std}} + (d_{\mathrm{std}} + 3) + (d_{\mathrm{std}} + d_\oplus) \\
    =&~ 4d_{\mathrm{std}} + 2d_\oplus + 3.    
\end{align*}
Since the circuits to compute each step are all in polynomial size, we can conclude that the MLP prediction head is computable with a $\poly(n)$ size and $(4d_{\mathrm{std}} + 2d_\oplus + 3)$ depth uniform threshold circuit. Hence, we finish our proof.
\end{proof} 

\subsection{Computing Multi-Layer Graph Neural Network}\label{sec:comp_gnn}
Since we have presented the circuit depths for all the components of an entire multi-layer GNN, we now derive the overall circuit depth of the entire GNN in this subsection. 

\begin{lemma}[Computing multi-layer GNN with $\mathsf
TC^0$ circuits]\label{lem:multi_gnn_layers}
    If $p\leq \poly(n), d= O(n)$, then the multi-layer graph neural network in Definition~\ref{dfn:gnn_multi_layer} can be computed with a $\poly(n)$ size uniform threshold circuit with depth $(md_{\mathrm{conv}} + (3m+4)d_{\mathrm{std}} + (2m+2)d_\oplus + d_{\mathrm{read}} + 3m+3)$. 
\end{lemma}
\begin{proof}
    For each layer $i\in[m]$, $\mathsf{GNN}_i$ can be computed with a $\poly(n)$ size and $(d_{\mathrm{conv}} + 3d_{\mathrm{std}} + 2d_{\oplus} + 3)$ depth uniform threshold circuit by Lemma~\ref{lem:one_gnn_layer}. Hence, the composition of all the $m$ layers $\mathsf{GNN}_m\circ \mathsf{GNN}_{m-1} \circ \cdots \circ \mathsf{GNN}_1(X)$ can be computed with a uniform threshold circuit within a total depth of $m(d_{\mathrm{conv}} + 3d_{\mathrm{std}} + 2d_{\oplus} + 3)$. Then, we can compute the graph readout layer with $\poly(n)$ size and $d_{\mathrm{read}}$ depth uniform threshold circuit, following Lemma~\ref{lem:pool_tc0} and Lemma~\ref{lem:maxpool_tc0}. Finally, we can apply Lemma~\ref{lem:pred_head} to compute the MLP prediction head with a $\poly(n)$ size and $(4d_{\mathrm{std}} + 2d_\oplus + 3)$ uniform threshold circuit. 

    Combining the circuits to compute each layer of the entire model, we have the following total circuit depth:
    \begin{align*}
    d_{\mathrm{total}} =&~ m(d_{\mathrm{conv}} + 3d_{\mathrm{std}} + 2d_{\oplus} + 3) + d_{\mathrm{read}} + (4d_{\mathrm{std}} + 2d_\oplus + 3)\\
    =&~ md_{\mathrm{conv}} + (3m+4)d_{\mathrm{std}} + (2m+2)d_\oplus + d_{\mathrm{read}} + 3m+3.
    \end{align*}
    Since the circuits for computing each part of the model are all in $\poly(n)$ size, we can conclude that the multi-layer GNN can be computed with a $\poly(n)$ size and $(md_{\mathrm{conv}} + (3m+4)d_{\mathrm{std}} + (2m+2)d_\oplus + d_{\mathrm{read}} + 3m+3)$ depth uniform threshold circuit. Hence, we finish the proof. 
\end{proof}

\subsection{Main Result: Graph Neural Networks Circuit Complexity }\label{sec:comp_result}
Finally, this subsection presents the circuit complexity bound of graph neural networks, which is the main result of this paper.

\begin{theorem}[Main Result, Circuit complexity bound of graph neural networks]\label{thm:complexity_gnn}
    If precision $p\leq\poly(n)$, embedding size $d= O(n)$ and the number of layers $m\leq O(1)$, then the graph neural network $\mathsf{GNN}:\mathbb{F}_p^{n\times d} \rightarrow \mathbb{F}_p$ in Definition~\ref{dfn:gnn_multi_layer} can be simulated by the uniform $\mathsf{TC}^0$ circuit family. 
\end{theorem}
\begin{proof}
    By Lemma~\ref{lem:multi_gnn_layers} and $m=O(1)$, the circuit to compute $\mathsf{GNN}(x)$ has depth
    $$md_{\mathrm{conv}} + (3m+4)d_{\mathrm{std}} + (2m+2)d_\oplus + d_{\mathrm{read}} + 3m+3 = O(1)$$
    and size $\poly(n)$. Then, we can conclude that a uniform $\mathsf{TC}^0$ circuit family can simulate $\mathsf{GNN}(x)$, following Definition~\ref{dfn:class_tci}. Hence, we finish the proof. 
\end{proof}

The main result in Theorem~\ref{thm:complexity_gnn} establishes that unless $\mathsf{TC}^0 = \mathsf{NC}^1$, graph neural networks with $\poly(n)$ precision, constant depth, and $O(n)$ embedding size belong to the uniform $\mathsf{TC}^0$ circuit class. This highlights an inherent expressiveness limitation of GNNs, despite their empirical success, as they cannot solve problems beyond the capability of $\mathsf{TC}^0$ circuits. In the next section, we illustrate this limitation by analyzing two practical graph query problems. 
\section{Hardness}\label{sec:hard}

We explore two critical decision problems on graphs and their associated hardness results. Section~\ref{sec:gconn} introduces two graph connectivity problems. Section~\ref{sec:gi} presents the basic concepts of the graph isomorphism problem. In Section~\ref{sec:hard_res}, we present the main hardness result of this paper, which theoretically demonstrates that graph neural networks cannot solve these problems.

\subsection{Graph Connectivity Problem}\label{sec:gconn}
To formally define the graph connectivity problem, we begin by introducing two basic concepts related to sequences of connected nodes in graphs: walks and paths.

\begin{definition}[Walk, Definition on page 26 of~\cite{wil10}]\label{dfn:walk}
    Given a graph $\mathcal{G} = (\mathcal{V}, \mathcal{E})$, a walk in $\mathcal{G}$ is a finite sequence of nodes, denoted by 
    $$v_0\rightarrow v_1 \rightarrow v_2\rightarrow\cdots\rightarrow v_m,$$
    where $v_0, v_1, v_2, \dots, v_m\in\mathcal{V}$ and any two consecutive nodes are connected (i.e., $\forall i\in[m], (v_{i-1}, v_i)\in\mathcal{E}$).
\end{definition}
\begin{definition}[Path, Definition on page 26 of~\cite{wil10}]\label{dfn:path}
    For a walk $v_0\rightarrow v_1 \rightarrow v_2\rightarrow\cdots\rightarrow v_m$, if all the nodes $v_0, \dots, v_m$ and all the edges $(v_0, v_1), \cdots, (v_{m-1}, v_m)$ are distinct, then we call this walk as a path. 
\end{definition}

Building upon these basic concepts, we now formally define two types of graph connectivity problems: the pairwise s-t connectivity problem, which checks whether two specific nodes are connected, and the global graph connectivity problem, which verifies the connectivity of all nodes.

\begin{definition}[Undirected graph s-t connectivity problem, Definition on page 2 of~\cite{wig92}]\label{dfn:ustconn}
    Let $\mathcal{G} = (\mathcal{V}, \mathcal{E})$ be an arbitrary undirected graph. The undirected s-t graph connectivity problem is: Does there exist a path between two specific nodes $s, t\in \mathcal{V}$?
\end{definition}
\begin{definition}[Undirected graph connectivity problem, Definition on page 2 of~\cite{wig92}]\label{dfn:uconn}
    Let $\mathcal{G} = (\mathcal{V}, \mathcal{E})$ be an arbitrary undirected graph. The undirected graph connectivity problem is: Does there exist a path between all the nodes in $\mathcal{V}$?
\end{definition}

With these problem formulations, we proceed to their computational complexity results.

\begin{lemma}[Theorem 2 in~\cite{wig92}]\label{lem:ustconn_complexity}
    The undirected graph s-t connectivity problem in Definition~\ref{dfn:ustconn} is $\mathsf{NC}^1$-complete. 
\end{lemma}

\begin{lemma}[Theorem 3 in~\cite{cm87}]\label{lem:uconn_complexity}
    The undirected graph connectivity problem in Definition~\ref{dfn:uconn} is $\mathsf{NC}^1$-hard. 
\end{lemma}

\subsection{Graph Isomorphism Problem}\label{sec:gi}

In this subsection, we turn to the graph isomorphism problem, a core challenge in understanding the expressiveness of GNNs~\cite{xhl18,zgd24}. This problem involves determining whether two graphs are structurally identical by examining possible permutations of their nodes. 

\begin{definition}[Graph isomorphism, on page 3 of~\cite{tor04}]\label{dfn:gi}
Let $\mathcal{G}_1 = (\mathcal{V}_1, \mathcal{E}_1)$ and $\mathcal{G}_2 = (\mathcal{V}_2, \mathcal{E}_2)$ be two graphs. An isomorphism between $\mathcal{G}_1$ and $\mathcal{G}_2$ is a bijection $\phi$ the between their sets of vertices $\mathcal{V}_1$ and $\mathcal{V}_2$ which preserves the edges, i.e. $\forall v_1, v_2 \in \mathcal{V}_1, (v_1, v_2)\in\mathcal{E}_1 ~\Leftrightarrow~ (\phi(v_1), \phi(v_2)) \in \mathcal{V}_2$ and $\forall v_1, v_2 \in \mathcal{V}_2, (v_1, v_2)\in\mathcal{E}_2 ~\Leftrightarrow~ (\phi(v_1), \phi(v_2)) \in \mathcal{V}_1$.
\end{definition}

We can now define the graph isomorphism problem, which checks the existence of such an isomorphism.

\begin{definition}[Graph isomorphism problem, on page 3 of~\cite{tor04}]\label{dfn:gi_prob}
    Let $\mathcal{G}_1 = (\mathcal{V}_1, \mathcal{E}_1)$ and $\mathcal{G}_2 = (\mathcal{V}_2, \mathcal{E}_2)$ be two graphs. The graph isomorphism problem is: Does there exist a graph isomorphism $\phi$ between $\mathcal{G}_1$ and $\mathcal{G}_2$?
\end{definition}

The computational complexity of the graph isomorphism problem is summarized in the following results:

\begin{lemma}[Theorem 4.9 in~\cite{tor04}]\label{cor:gi_complexity_det}
    The graph isomorphism problem in Definition~\ref{dfn:gi_prob} is hard for the class $\mathsf{DET}$ under $\mathsf{AC}^0$ reductions. 
\end{lemma}

\begin{corollary}\label{cor:gi_complexity_tc}
The graph isomorphism problem in Definition~\ref{dfn:gi_prob} is $\mathsf
NC^1$-hard. 
\end{corollary}
\begin{proof}
    By Lemma~\ref{cor:gi_complexity_det}, it is known that the graph isomorphism problem is hard for the class $\mathsf{DET}$ under $\mathsf{AC}^0$ reductions. Following Fact~\ref{fact:complexity_relation}, since $\mathsf{AC}^0\subseteq\mathsf{NC}^1\subseteq\mathsf{DET}$, we can conclude that the graph isomorphism problem is $\mathsf{DET}$-hard ignoring the $\mathsf{AC}^0$ reduction condition. Thus, we can apply Fact~\ref{fact:complexity_relation} for the second time and conclude that the graph isomorphism problem is $\mathsf{NC}^1$-hard, which finishes the proof.
\end{proof}

\subsection{Hardness Results}\label{sec:hard_res}
This subsection presents the main hardness results, which highlight the limitations of GNNs on two practical graph decision problems. We begin by showing that GNNs are unable to solve both types of graph connectivity problems.

\begin{theorem}[]\label{thm:ustconn_hard}
    Unless $\mathsf{TC}^0 = \mathsf{NC}^1$, a graph neural network with $\poly(n)$ precision, constant number of layers, embedding size $d= O(n)$ cannot solve the graph s-t connectivity problem.
\end{theorem}
\begin{proof}
    By Theorem~\ref{thm:complexity_gnn}, we have already shown that the graph neural network is in the $\mathsf{TC}^0$ circuit family. We can also conclude that the graph s-t connectivity problem is in $\mathsf{NC}^1$ by Lemma~\ref{lem:ustconn_complexity}. Thus, combining with Fact~\ref{fact:complexity_relation}, we can complete the proof. 
\end{proof}

\begin{theorem}[]\label{thm:uconn_hard}
        Unless $\mathsf{TC}^0 = \mathsf{NC}^1$, a graph neural network with $\poly(n)$ precision, constant number of layers, embedding size $d= O(n)$ cannot solve the graph connectivity problem.
\end{theorem}
\begin{proof}
    By Theorem~\ref{thm:complexity_gnn}, we have already shown that the graph neural network is in the $\mathsf{TC}^0$ circuit family. We can also conclude that the graph connectivity problem is $\mathsf{NC}^1$-hard by Lemma~\ref{lem:uconn_complexity}. Thus, combining with Fact~\ref{fact:complexity_relation}, we can complete the proof. 
\end{proof}

Next, we show the hardness results for the graph isomorphism problem, which demonstrates the expressiveness limitations of GNNs from a non-Weisfeiler-Lehman (WL) perspective.
\begin{theorem}[]\label{thm:gi_hard}
    Unless $\mathsf{TC}^0 = \mathsf{NC}^1$, a graph neural network with $\poly(n)$ precision, constant number of layers, embedding size $d= O(n)$ cannot solve the graph isomorphism problem.
\end{theorem}
\begin{proof}
    By Theorem~\ref{thm:complexity_gnn}, we have already shown that the graph neural network is in the $\mathsf{TC}^0$ circuit family. We can also conclude that the graph connectivity problem is $\mathsf{NC}^1$-hard by Corollary~\ref{cor:gi_complexity_tc}. Thus, combining with Fact~\ref{fact:complexity_relation}, we can complete the proof. 
\end{proof}

\begin{remark}
    Our hardness results assume an embedding size of $d = O(n)$, which is significantly stronger and encompasses the constant embedding sizes typically used in practice, where $d = O(1)$. This highlights that the computational limitations of GNNs cannot be easily mitigated by simply increasing the embedding size.  
\end{remark}

\subsection{Discussion}\label{sec:hard_discuss}

After presenting the main hardness results of this paper, we explore the connections and comparisons with a broad range of relevant works, highlighting the contribution of this paper.

\paragraph{Connection and comparison to WL.}
The Weisfeiler-Lehman (WL) hierarchy is the most widely used framework for analyzing the expressiveness of GNNs, particularly for their ability to solve graph isomorphism problems~\cite{xhl18,zgd24}. While our Theorem~\ref{thm:gi_hard} similarly shows that GNNs cannot solve the graph isomorphism problem unless $\mathsf{TC}^0 = \mathsf{NC}^1$, our results differ fundamentally from previous WL-based findings. Unlike the WL framework, which focuses solely on the topological structure of graphs, our analysis considers practical factors such as floating-point precision and node features. This broader perspective allows us to bound the expressiveness of feature-enhanced GNNs, including those with position encodings~\cite{yyl19,dll22} or random node features~\cite{syk21}. Moreover, while the WL framework is limited to graph isomorphism, we extend the analysis to show that GNNs cannot solve the graph connectivity problem (Theorem~\ref{thm:ustconn_hard} and Theorem~\ref{thm:uconn_hard}), a limitation that is difficult to capture using the WL framework.

\paragraph{Comparison to GNN computational models.} 
Recent studies have explored GNN expressiveness through computational models. For example, \cite{cws24} investigates the RL-CONGEST model, which examines the communication complexity of GNNs in distributed settings. However, this work focuses on communication cost rather than expressiveness and is less relevant to our study. The most similar work to ours is~\cite{bkm20}, which shows that aggregate-combine GNNs cannot capture a variant of first-order logic, $\mathsf{FOC}_2$. Since counting in $\mathsf{FOC}_2$ can be simulated by $\mathsf{FO}$ and $\mathsf{FO} = \mathsf{AC}^0$, their result implies that $\mathsf{AC}^0$ is not a complexity lower bound for AC-GNNs. In contrast, our work establishes that GNNs are upper-bounded by the $\mathsf{TC}^0$ circuit family, providing a clear upper bound rather than refusing a lower bound. This highlights a fundamental difference between our result and previous results in~\cite{bkm20}.
\section{Conclusion}\label{sec:conclusion}
In this work, we show the computational limits of graph neural networks (GNNs). Unlike prior approaches based on the Weisfeiler-Lehman (WL) test, we adopt a fundamentally different perspective: circuit complexity. We show that GNNs with $\poly(n)$ precision, constant number of layers, and embedding sizes $d = O(n)$, regardless of the specific message-passing mechanisms or global readout functions, belong to the uniform $\mathsf{TC}^0$ circuit complexity class. As a result, we establish critical hardness results for two practical graph decision problems: graph connectivity and graph isomorphism. Our findings demonstrate that GNNs cannot solve these problems beyond $\mathsf{TC}^0$, unless $\mathsf{TC}^0 = \mathsf{NC}^1$.
These results are significant as they extend previous GNN expressiveness limitations, which were framed from the WL perspective, by introducing a fundamentally different circuit complexity bound. Our analysis not only incorporates previously overlooked factors, such as floating-point number precision and the interactions between node embeddings and topological structure but also applies to a broader range of graph query problems, such as graph connectivity.

For future works, we believe our theoretical framework offers a trustworthy foundation that can be generalized to assess the expressiveness of other GNN architectures and the hardness of additional graph query problems. Our research may also inspire the development of new architectures that go beyond the complexity of the $\mathsf{TC}^0$ circuit family.

\ifdefined\isarxiv
\bibliographystyle{alpha}
\bibliography{ref}
\else
\bibliographystyle{named}
\bibliography{ref}

\newcommand{\etalchar}[1]{$^{#1}$}
\begin{thebibliography}{WWCBF22}

\bibitem[AB09]{ab09}
Sanjeev Arora and Boaz Barak.
\newblock {\em Computational complexity: a modern approach}.
\newblock Cambridge University Press, 2009.

\bibitem[Bab16]{bab16}
L{\'a}szl{\'o} Babai.
\newblock Graph isomorphism in quasipolynomial time.
\newblock In {\em Proceedings of the forty-eighth annual ACM symposium on Theory of Computing}, pages 684--697, 2016.

\bibitem[BAY22]{bay22}
Shaked Brody, Uri Alon, and Eran Yahav.
\newblock How attentive are graph attention networks?
\newblock In {\em International Conference on Learning Representations}, 2022.

\bibitem[BI94]{bi94}
D~Mix Barrington and Neil Immerman.
\newblock Time, hardware, and uniformity.
\newblock In {\em Proceedings of IEEE 9th Annual Conference on Structure in Complexity Theory}, pages 176--185. IEEE, 1994.

\bibitem[BKM{\etalchar{+}}20]{bkm20}
Pablo Barceló, Egor~V. Kostylev, Mikael Monet, Jorge Pérez, Juan Reutter, and Juan~Pablo Silva.
\newblock The logical expressiveness of graph neural networks.
\newblock In {\em International Conference on Learning Representations}, 2020.

\bibitem[BL23]{bl23}
Filippo~Maria Bianchi and Veronica Lachi.
\newblock The expressive power of pooling in graph neural networks.
\newblock {\em Advances in neural information processing systems}, 36:71603--71618, 2023.

\bibitem[BRH{\etalchar{+}}21]{brh21}
Muhammet Balcilar, Guillaume Renton, Pierre H{\'e}roux, Benoit Ga{\"u}z{\`e}re, S{\'e}bastien Adam, and Paul Honeine.
\newblock Analyzing the expressive power of graph neural networks in a spectral perspective.
\newblock In {\em International Conference on Learning Representations}, 2021.

\bibitem[Chi24]{chi24}
David Chiang.
\newblock Transformers in uniform {$\mathsf{TC}^0$}.
\newblock {\em arXiv preprint arXiv:2409.13629}, 2024.

\bibitem[CLL{\etalchar{+}}24a]{cll+24}
Bo~Chen, Xiaoyu Li, Yingyu Liang, Jiangxuan Long, Zhenmei Shi, and Zhao Song.
\newblock Circuit complexity bounds for rope-based transformer architecture.
\newblock {\em arXiv preprint arXiv:2411.07602}, 2024.

\bibitem[CLL{\etalchar{+}}24b]{cll24}
Yifang Chen, Xiaoyu Li, Yingyu Liang, Zhenmei Shi, and Zhao Song.
\newblock The computational limits of state-space models and mamba via the lens of circuit complexity.
\newblock {\em arXiv preprint arXiv:2412.06148}, 2024.

\bibitem[CM87]{cm87}
Stephen~A Cook and Pierre McKenzie.
\newblock Problems complete for deterministic logarithmic space.
\newblock {\em Journal of Algorithms}, 8(3):385--394, 1987.

\bibitem[CMR21]{cmr21}
Leonardo Cotta, Christopher Morris, and Bruno Ribeiro.
\newblock Reconstruction for powerful graph representations.
\newblock {\em Advances in Neural Information Processing Systems}, 34:1713--1726, 2021.

\bibitem[Coo85]{coo85}
Stephen~A Cook.
\newblock A taxonomy of problems with fast parallel algorithms.
\newblock {\em Information and control}, 64(1-3):2--22, 1985.

\bibitem[CWS24]{cws24}
Guanyu Cui, Zhewei Wei, and Hsin-Hao Su.
\newblock Rethinking the expressiveness of gnns: A computational model perspective.
\newblock {\em arXiv preprint arXiv:2410.01308}, 2024.

\bibitem[DLL{\etalchar{+}}22]{dll22}
Vijay~Prakash Dwivedi, Anh~Tuan Luu, Thomas Laurent, Yoshua Bengio, and Xavier Bresson.
\newblock Graph neural networks with learnable structural and positional representations.
\newblock In {\em International Conference on Learning Representations}, 2022.

\bibitem[FCL{\etalchar{+}}22]{fcl22}
Jiarui Feng, Yixin Chen, Fuhai Li, Anindya Sarkar, and Muhan Zhang.
\newblock How powerful are k-hop message passing graph neural networks.
\newblock {\em Advances in Neural Information Processing Systems}, 35:4776--4790, 2022.

\bibitem[FML{\etalchar{+}}19]{fml19}
Wenqi Fan, Yao Ma, Qing Li, Yuan He, Eric Zhao, Jiliang Tang, and Dawei Yin.
\newblock Graph neural networks for social recommendation.
\newblock In {\em The world wide web conference}, pages 417--426, 2019.

\bibitem[GD23]{gd23}
Albert Gu and Tri Dao.
\newblock Mamba: Linear-time sequence modeling with selective state spaces.
\newblock {\em arXiv preprint arXiv:2312.00752}, 2023.

\bibitem[GJJ20]{gjj20}
Vikas Garg, Stefanie Jegelka, and Tommi Jaakkola.
\newblock Generalization and representational limits of graph neural networks.
\newblock In {\em International Conference on Machine Learning}, pages 3419--3430. PMLR, 2020.

\bibitem[GS20]{gs20}
Martin Grohe and Pascal Schweitzer.
\newblock The graph isomorphism problem.
\newblock {\em Communications of the ACM}, 63(11):128--134, 2020.

\bibitem[GSR{\etalchar{+}}17]{gsr17}
Justin Gilmer, Samuel~S Schoenholz, Patrick~F Riley, Oriol Vinyals, and George~E Dahl.
\newblock Neural message passing for quantum chemistry.
\newblock In {\em International conference on machine learning}, pages 1263--1272. PMLR, 2017.

\bibitem[HAF22]{haf22}
Yiding Hao, Dana Angluin, and Robert Frank.
\newblock Formal language recognition by hard attention transformers: Perspectives from circuit complexity.
\newblock {\em Transactions of the Association for Computational Linguistics}, 10:800--810, 2022.

\bibitem[HDW{\etalchar{+}}20]{hdw20}
Xiangnan He, Kuan Deng, Xiang Wang, Yan Li, Yongdong Zhang, and Meng Wang.
\newblock Lightgcn: Simplifying and powering graph convolution network for recommendation.
\newblock In {\em Proceedings of the 43rd International ACM SIGIR conference on research and development in Information Retrieval}, pages 639--648, 2020.

\bibitem[HYL17]{hyl17}
Will Hamilton, Zhitao Ying, and Jure Leskovec.
\newblock Inductive representation learning on large graphs.
\newblock {\em Advances in neural information processing systems}, 30, 2017.

\bibitem[KW17]{kw17}
Thomas~N. Kipf and Max Welling.
\newblock Semi-supervised classification with graph convolutional networks.
\newblock In {\em International Conference on Learning Representations}, 2017.

\bibitem[LAG{\etalchar{+}}23]{lag23}
Bingbin Liu, Jordan~T. Ash, Surbhi Goel, Akshay Krishnamurthy, and Cyril Zhang.
\newblock Transformers learn shortcuts to automata.
\newblock In {\em The Eleventh International Conference on Learning Representations}, 2023.

\bibitem[LLL{\etalchar{+}}24]{lly24}
Xiaoyu Li, Yuanpeng Li, Yingyu Liang, Zhenmei Shi, and Zhao Song.
\newblock On the expressive power of modern hopfield networks.
\newblock {\em arXiv preprint arXiv:2412.05562}, 2024.

\bibitem[LLS{\etalchar{+}}24]{lls24}
Xiaoyu Li, Yingyu Liang, Zhenmei Shi, Zhao Song, and Mingda Wan.
\newblock Theoretical constraints on the expressive power of rope-based tensor attention transformers.
\newblock {\em arXiv preprint arXiv:2412.18040}, 2024.

\bibitem[LW68]{lw68}
Andrei Leman and Boris Weisfeiler.
\newblock A reduction of a graph to a canonical form and an algebra arising during this reduction.
\newblock {\em Nauchno-Technicheskaya Informatsiya}, 2(9):12--16, 1968.

\bibitem[MBHSL19]{mbh19}
Haggai Maron, Heli Ben-Hamu, Hadar Serviansky, and Yaron Lipman.
\newblock Provably powerful graph networks.
\newblock {\em Advances in neural information processing systems}, 32, 2019.

\bibitem[MRF{\etalchar{+}}19]{mrf19}
Christopher Morris, Martin Ritzert, Matthias Fey, William~L Hamilton, Jan~Eric Lenssen, Gaurav Rattan, and Martin Grohe.
\newblock Weisfeiler and leman go neural: Higher-order graph neural networks.
\newblock In {\em Proceedings of the AAAI conference on artificial intelligence}, volume~33, pages 4602--4609, 2019.

\bibitem[MRM20]{mrm20}
Christopher Morris, Gaurav Rattan, and Petra Mutzel.
\newblock Weisfeiler and leman go sparse: Towards scalable higher-order graph embeddings.
\newblock {\em Advances in Neural Information Processing Systems}, 33:21824--21840, 2020.

\bibitem[MS23]{ms23}
William Merrill and Ashish Sabharwal.
\newblock The parallelism tradeoff: Limitations of log-precision transformers.
\newblock {\em Transactions of the Association for Computational Linguistics}, 11:531--545, 2023.

\bibitem[MS24]{ms24}
William Merrill and Ashish Sabharwal.
\newblock A logic for expressing log-precision transformers.
\newblock {\em Advances in Neural Information Processing Systems}, 36, 2024.

\bibitem[MSS22]{mss22}
William Merrill, Ashish Sabharwal, and Noah~A Smith.
\newblock Saturated transformers are constant-depth threshold circuits.
\newblock {\em Transactions of the Association for Computational Linguistics}, 10:843--856, 2022.

\bibitem[PW22]{pw22}
P{\'a}l~Andr{\'a}s Papp and Roger Wattenhofer.
\newblock A theoretical comparison of graph neural network extensions.
\newblock In {\em International Conference on Machine Learning}, pages 17323--17345. PMLR, 2022.

\bibitem[QRG{\etalchar{+}}22]{qrg22}
Chendi Qian, Gaurav Rattan, Floris Geerts, Mathias Niepert, and Christopher Morris.
\newblock Ordered subgraph aggregation networks.
\newblock {\em Advances in Neural Information Processing Systems}, 35:21030--21045, 2022.

\bibitem[Ruz81]{ruz81}
Walter~L Ruzzo.
\newblock On uniform circuit complexity.
\newblock {\em Journal of Computer and System Sciences}, 22(3):365--383, 1981.

\bibitem[SAL{\etalchar{+}}24]{sal24}
Jianlin Su, Murtadha Ahmed, Yu~Lu, Shengfeng Pan, Wen Bo, and Yunfeng Liu.
\newblock Roformer: Enhanced transformer with rotary position embedding.
\newblock {\em Neurocomputing}, 568:127063, 2024.

\bibitem[Sat20]{sat20}
Ryoma Sato.
\newblock A survey on the expressive power of graph neural networks.
\newblock {\em arXiv preprint arXiv:2003.04078}, 2020.

\bibitem[SLYS21]{sly21}
Aravind Sankar, Yozen Liu, Jun Yu, and Neil Shah.
\newblock Graph neural networks for friend ranking in large-scale social platforms.
\newblock In {\em Proceedings of the Web Conference 2021}, pages 2535--2546, 2021.

\bibitem[SYK21]{syk21}
Ryoma Sato, Makoto Yamada, and Hisashi Kashima.
\newblock Random features strengthen graph neural networks.
\newblock In {\em Proceedings of the 2021 SIAM international conference on data mining (SDM)}, pages 333--341. SIAM, 2021.

\bibitem[Tor04]{tor04}
Jacobo Tor{\'a}n.
\newblock On the hardness of graph isomorphism.
\newblock {\em SIAM Journal on Computing}, 33(5):1093--1108, 2004.

\bibitem[VCC{\etalchar{+}}18]{vcc18}
Petar Veličković, Guillem Cucurull, Arantxa Casanova, Adriana Romero, Pietro Liò, and Yoshua Bengio.
\newblock Graph attention networks.
\newblock In {\em International Conference on Learning Representations}, 2018.

\bibitem[Vol99]{vol99}
Heribert Vollmer.
\newblock {\em Introduction to circuit complexity: a uniform approach}.
\newblock Springer Science \& Business Media, 1999.

\bibitem[WHW{\etalchar{+}}19]{wxw19}
Xiang Wang, Xiangnan He, Meng Wang, Fuli Feng, and Tat-Seng Chua.
\newblock Neural graph collaborative filtering.
\newblock In {\em Proceedings of the 42nd international ACM SIGIR conference on Research and development in Information Retrieval}, pages 165--174, 2019.

\bibitem[Wig92]{wig92}
Avi Wigderson.
\newblock The complexity of graph connectivity.
\newblock In {\em International Symposium on Mathematical Foundations of Computer Science}, pages 112--132. Springer, 1992.

\bibitem[Wil10]{wil10}
R.J. Wilson.
\newblock {\em Introduction to Graph Theory}.
\newblock Longman, 2010.

\bibitem[WSZ{\etalchar{+}}19]{wsz19}
Felix Wu, Amauri Souza, Tianyi Zhang, Christopher Fifty, Tao Yu, and Kilian Weinberger.
\newblock Simplifying graph convolutional networks.
\newblock In {\em International conference on machine learning}, pages 6861--6871. PMLR, 2019.

\bibitem[WW22]{ww22}
Asiri Wijesinghe and Qing Wang.
\newblock A new perspective on ''how graph neural networks go beyond weisfeiler-lehman?''.
\newblock In {\em International Conference on Learning Representations}, 2022.

\bibitem[WWCBF22]{wwc22}
Yuyang Wang, Jianren Wang, Zhonglin Cao, and Amir Barati~Farimani.
\newblock Molecular contrastive learning of representations via graph neural networks.
\newblock {\em Nature Machine Intelligence}, 4(3):279--287, 2022.

\bibitem[XHLJ18]{xhl18}
Keyulu Xu, Weihua Hu, Jure Leskovec, and Stefanie Jegelka.
\newblock How powerful are graph neural networks?
\newblock In {\em International Conference on Learning Representations}, 2018.

\bibitem[YYL19]{yyl19}
Jiaxuan You, Rex Ying, and Jure Leskovec.
\newblock Position-aware graph neural networks.
\newblock In {\em International conference on machine learning}, pages 7134--7143. PMLR, 2019.

\bibitem[ZC18]{zc18}
Muhan Zhang and Yixin Chen.
\newblock Link prediction based on graph neural networks.
\newblock {\em Advances in neural information processing systems}, 31, 2018.

\bibitem[ZCNC18]{zcn18}
Muhan Zhang, Zhicheng Cui, Marion Neumann, and Yixin Chen.
\newblock An end-to-end deep learning architecture for graph classification.
\newblock In {\em Proceedings of the AAAI conference on artificial intelligence}, volume~32, 2018.

\bibitem[ZFD{\etalchar{+}}23]{zfd23}
Bohang Zhang, Guhao Feng, Yiheng Du, Di~He, and Liwei Wang.
\newblock A complete expressiveness hierarchy for subgraph gnns via subgraph weisfeiler-lehman tests.
\newblock In {\em International Conference on Machine Learning}, pages 41019--41077. PMLR, 2023.

\bibitem[ZGD{\etalchar{+}}24]{zgd24}
Bohang Zhang, Jingchu Gai, Yiheng Du, Qiwei Ye, Di~He, and Liwei Wang.
\newblock Beyond weisfeiler-lehman: A quantitative framework for {GNN} expressiveness.
\newblock In {\em The Twelfth International Conference on Learning Representations}, 2024.

\bibitem[ZZXT21]{zzx21}
Zhaocheng Zhu, Zuobai Zhang, Louis-Pascal Xhonneux, and Jian Tang.
\newblock Neural bellman-ford networks: A general graph neural network framework for link prediction.
\newblock {\em Advances in Neural Information Processing Systems}, 34:29476--29490, 2021.

\end{thebibliography}

\fi

\newpage
\onecolumn
\appendix




\end{document}